\documentclass[11pt,a4paper]{scrartcl}
\usepackage{graphicx}
\usepackage{subfig}
\usepackage{amsfonts}
\usepackage{amsthm}
\usepackage{amssymb}
\usepackage{amstext}
\usepackage{amsmath}
\usepackage{amsthm}
\usepackage{pstricks}
\usepackage{pspicture}
\usepackage{amscd}
\usepackage{appendix}
\usepackage{natbib}

\usepackage[margin=2.5cm]{geometry}
\selectfont

 \newtheorem{proposition}{Proposition}
  
 \newtheorem{definition}{Definition}

    \newtheorem{theorem}{Theorem}
  \newtheorem{property}{Axiom}

  \theoremstyle{definition}
  \newtheorem{example}{Example}
\newcommand{\dsum}{\displaystyle\sum}
\begin{document}
\title{\bf Axiomatic properties of inconsistency indices for pairwise comparisons}
\author{
{\bf Matteo Brunelli}
\\
{\normalsize Systems Analysis Laboratory, Department of Mathematics and Systems Analysis} \\
{\normalsize Aalto University}, {\normalsize P.O. box 11100,
FIN-00076 Aalto, Finland}
\\ {\normalsize e--mail:
\texttt{matteo.brunelli@aalto.fi}}
\vspace{0.3cm}\\
{\bf Michele Fedrizzi}
\\
{\normalsize Department of Industrial Engineering} \\
{\normalsize University of Trento}, {\normalsize Via Mesiano 77, I-38123 Trento, Italy}
\\ {\normalsize e--mail:
\texttt{michele.fedrizzi@unitn.it}}
}
\date{}

\maketitle \thispagestyle{empty}


\begin{center}
{\bf Abstract}
\end{center}

{\small \noindent 
Pairwise comparisons are a well-known method for the representation
of the subjective preferences of a decision maker. Evaluating their
inconsistency has been a widely studied and discussed topic and
several indices have been proposed in the literature to perform this
task. Since an acceptable level of consistency is closely related
with the reliability of preferences, a suitable choice of an
inconsistency index is a crucial phase in decision making processes.
The use of different methods for measuring consistency must be
carefully evaluated, as it can affect the decision outcome in
practical applications. In this paper, we present five axioms aimed
at characterizing inconsistency indices. In addition, we prove that
some of the indices proposed in the literature satisfy these axioms,
while others do not, and therefore, in our view, they may fail to
correctly evaluate inconsistency.}

 \vspace{0.3cm}
 \vspace{0.3cm}

\section{Introduction}

Pairwise comparisons have been used in some operations research methods to represent the preferences of experts and decision
makers over sets of alternatives, criteria, features and so on. For simplicity, in this paper we shall speak of alternatives only, bearing in mind that it is a reductive view. The main advantage in using pairwise comparisons is that they allow the
decision maker to compare two alternatives at a time, thus
reducing the complexity of a decision making problem, especially
when the set under consideration is large, and 
serve as a starting point to derive a priority vector which is the final rating of the alternatives. Pairwise comparisons have been used in well-known decision analysis methods as, for instance, the Analytic Hierarchy Process (AHP) by \citet{Saaty1977} (see \citet{IshizakaLabib2011} for an updated discussion), and its generalizations, which have been proved effective in solving many decision problems \citep{IshizakaEtAl2011}. 


In the literature, and in practice, it is assumed that the
dependability of the decision is related to the consistency
of his/her pairwise judgments. 
That is, the more
rational the judgments are, the more likely it is that the
decision maker is a good expert with a deep insight into the
problem and pays due attention in eliciting his/her
preferences. Similarly, if judgments are very
intransitive and irrational, it is more plausible that the expert expressed them with
scarce competence, since he/she would lack the ability to rationally discriminate between different alternatives. This is summarized by Irwin's thesis claiming that ``...preference is exactly as fundamental as discrimination and that if the organism exhibits a discrimination, it must also exhibit a preference and conversely" \citep{Irwin1958}.
Following \citet{Saaty1994} the approach to decision making based on pairwise
comparisons, and the AHP in particular, is grounded in the relative measurement theory and it is in this framework that \citet{Saaty1993} too claimed that pairwise comparisons should be `near consistent' to ensure that they are
a sufficiently good approximation of the decision makers' real preferences. This seems to support the importance of having reliable tools capable of capturing the degree of inconsistency of pairwise comparisons. The importance of having reliable inconsistency indices becomes even more evident when one considers that their practical use has gone beyond the sole quantification of inconsistency. For instance, they have been employed by \citet{LamataPelaez2002} and \citet{ShiraishiEtAl1999} to estimate missing comparisons, by \citet{Harker1987} to derive ratings of alternatives from incomplete preferences and by \citet{XuCuiping1999} and \citet{XuXia2013} to improve the consistency of pairwise comparisons.\\

On this fertile ground, researchers  have proposed various inconsistency indices---functions associating pairwise comparisons to real numbers representing the degrees of inconsistency of the pairwise judgments.
In this paper we concern ourselves with the fact that inconsistency
indices have been introduced heuristically and independently from
each other, neither referring to a general definition, nor to a set
of axiomatic properties.
Hence, this paper introduces some axiomatic properties for inconsistency
indices and shows that some indices proposed in literature fail to
satisfy these axioms.
This paper is outlined as follows.
In Section \ref{sec:preliminaries} we introduce preliminary notions and the notation.
In Section \ref{sec:indices} we shortly define the inconsistency indices that are studied in this paper.
Next, in Section \ref{sec:properties} we introduce and interpret five axioms and in Section \ref{sec:results} we present results
regarding the inconsistency indices and prove that some of them satisfy the required axioms while four others do not.
For a simpler description, some proofs are given in the
appendix. In Section \ref{sec:open_question}, we conclude the discussion of the axioms and draw the conclusions.\\
Throughout this paper, we refer to `inconsistency indices', since
what they really measure is the amount of inconsistency in pairwise
comparisons. Nevertheless, in literature such indices are often
referred to as `consistency indices', while both expressions refer to an index
which estimates the deviation from consistency.

\section{Preliminaries}
\label{sec:preliminaries}

Pairwise comparison matrices are convenient tools to model the decision makers' pairwise intensities of preference over sets of alternatives.
Formally, given a set of alternatives $X=\{ x_{1},\ldots,x_{n}
\}~(n \geq 2)$, \citet{Saaty1977} defined a \emph{pairwise comparison matrix}
$\mathbf{A}=(a_{ij})_{n \times n}$ as a positive and reciprocal square
matrix of order $n$, i.e. $a_{ij}>0,~a_{ij}a_{ji}=1,~\forall i,j$,
where $a_{ij}$ is an estimation of the degree of preference of
$x_i$ over $x_j$. A pairwise comparison matrix
is {\em consistent} if and only if the following
transitivity condition holds:
\begin{equation}
\label{eq:transitivity}
a_{ik}=a_{ij}a_{jk}~  \forall i,j,k.
\end{equation}
\noindent Property (\ref{eq:transitivity}) means that preferences
are fully coherent, and each direct comparison
$a_{ik}$ between $x_i$ and $x_k$  is confirmed by all indirect
comparisons $a_{ij} a_{jk} \; \forall j$. If and
only if $\mathbf{A}$ is consistent, then there exists a priority (or weight)
vector $\mathbf{w}=(w_{1},\ldots,w_{n})$ such that
\begin{equation}
\label{eq:ratio}
a_{ij}=\frac{w_i}{w_j}~  \forall i,j.
\end{equation}
\citet{CrawfordWilliams1985} proved that if $\mathbf{A}$ is consistent, then the components of vector
$\mathbf{w}$ can be obtained by using the geometric mean method,
\begin{equation}
\label{eq:mediageometrica}
w_{i}=\left( \prod_{j=1}^{n}a_{ij} \right)^{\frac{1}{n}}~  \forall i .
\end{equation}
Another method for obtaining the priorities is the eigenvector method by \citet{Saaty1977}. Namely, the priority vector $\mathbf{w}$ is the solution of the following equation
\begin{equation}
\label{eq:eigenvector}
\mathbf{A}\mathbf{w}=\lambda_{\max}\mathbf{w} \, ,
\end{equation}
where $\lambda_{\max}$ is the maximum eigenvalue of $\mathbf{A}$
whose existence and properties refer to the Perron-Frobenius
theorem. If $\mathbf{A}$ is consistent, both methods yield the same
priority vector, while they may give different vectors if
$\mathbf{A}$ is not consistent.\\
We define the set of all pairwise
comparison matrices as
\[
\mathcal{A}= \{ \mathbf{A}=(a_{ij})_{n \times n} | a_{ij}>0, a_{ij}a_{ji}=1 ~\forall i,j, ~ n>2 \}.
\]
Similarly, the set of \emph{consistent} pairwise comparison matrices $\mathcal{A}^{*} \subset  \mathcal{A}$ is defined as
\[
\mathcal{A}^{*}= \{ \mathbf{A}=(a_{ij})_{n \times n} | \mathbf{A} \in \mathcal{A}, a_{ik} = a_{ij}a_{jk} ~ \forall i,j,k \}.
\]
Seen from this perspective, a matrix can either be consistent or
non-consistent (inconsistent). However, often, \emph{degrees of
inconsistency} are assigned to pairwise comparison matrices so that, if the inconsistency is not too high,
the judgments in the pairwise comparison matrix are taken to be as sufficiently reliable. To sum up, the idea is
that a good inconsistency index should indicate `how much' the pairwise comparison matrix deviates from the full
consistency. 
Thus, an inconsistency
index $I$ is a real-valued function
\begin{equation}
\label{eq:inconsistency_index}
I: \mathcal{A} \rightarrow \mathbb{R}.
\end{equation}
Although the codomain of the definition is the set of real numbers,
each inconsistency index $I$ is univocally associated to a given image $\text{Im}(I) \subseteq \mathbb{R}$.

\section{Inconsistency indices}
\label{sec:indices}

In this section we shortly recall some inconsistency indices, giving a special emphasis to those which will be analyzed in the next section with respect to the five axioms. For a survey the reader can refer to \citet{BrunelliCanalFedrizzi}.
The first index is the \emph{Consistency Index}, proposed by \citet{Saaty1977}.
\begin{definition}[Consistency Index \citep{Saaty1977}]
\label{def:CI}
Given a pairwise comparison matrix $\mathbf{A}$, the Consistency Index is defined as
\begin{equation}
\label{eq:CI}
CI(\mathbf{A})=\frac{\lambda_{\max}-n}{n-1},
\end{equation}
where $\lambda_{\max}$ is the principal right eigenvalue of
$\mathbf{A}$.
\end{definition}
\noindent Formula (\ref{eq:CI}) refers to the property that
the maximum eigenvalue $\lambda_{\max}$ of a pairwise comparison matrix $\mathbf{A}$ is equal to $n$ if and only if the matrix is consistent, and greater than $n$ otherwise. Saaty proposed also
a more suitable measure of inconsistency, called Consistency Ratio
(CR),
\begin{equation}
\label{eq:CR}
CR(\mathbf{A})=\frac{CI(\mathbf{A})}{RI}
\end{equation}
where $RI$, Random Index, is a suitable normalization factor.

\citet{GoldenWang1989} proposed a method to compute the deviations
between the entries of  a pairwise comparison matrix and their
theoretical values $w_{i}/w_{j}$.
\begin{definition}[Index $GW$ \citep{GoldenWang1989}]
\label{def:GW}
Given a pairwise comparison matrix $\mathbf{A} = (a_{ij})_{n
\times n} \in \mathcal{A}$, the entries of every column are
normalized by dividing them by the sum of the elements of their
column $\sum_{i=1}^{n}a_{ij}$. Let us denote by $\bar{\mathbf{A}}
= (\bar{a}_{ij})_{n \times n}$ the new normalized matrix. Each
priority vector associated (by either (\ref{eq:mediageometrica}) or (\ref{eq:eigenvector})) with $\mathbf{A}$ is normalized by
dividing each component by the sum of the components and denoted
by $\bar{\mathbf{w}}=(\bar{w}_{1},\ldots,\bar{w}_{n})$, so that
$\sum_{i=1}^{n} \bar{w}_{i} =1$. The inconsistency index $GW$ is
defined as
\begin{equation}
\label{eq:GW} GW(\mathbf{A})=\frac{1}{n}
\sum_{i=1}^{n}\sum_{j=1}^{n}| \, \bar{a}_{ij}-\bar{w}_{i}| \, .
\end{equation}
\end{definition}

\citet{CrawfordWilliams1985}, and later
\citet{AguaronMoreno2003}, proposed and refined an index that also
computes distances between the decision maker's judgments and
their theoretical values obtained as ratios $w_{i}/w_{j}$.
\begin{definition}[Geometric Consistency Index \citep{AguaronMoreno2003}]
\label{def:GCI}
Given a pairwise comparison matrix of order $n$, the Geometric Consistency Index $GCI$ is defined as follows
\begin{equation}
\label{eq:GCI}
GCI(\mathbf{A})=\frac{2}{(n-1)(n-2)}\sum_{i=1}^{n-1}\sum_{j=i+1}^{n} \,
\ln^2 \left( {a_{ij}\frac{w_j}{w_i}} \right) \, ,
\end{equation}
where the weights are obtained by means of the geometric mean method (\ref{eq:mediageometrica}).
\end{definition}

\citet{Barzilai1998} formulated a normalized index based on
squared errors. By using open unbounded scales, he stated several relevant algebraic and geometric properties.
\begin{definition}[Relative Error \citep{Barzilai1998}]
\label{def:RE}
Given a pairwise comparison matrix $\mathbf{A}\in\mathcal{A}$, the relative error, $RE$, is defined as
\begin{equation}
\label{eq:RE}
RE(\mathbf{A})=1-\dfrac{ \dsum_{i=1}^{n} \sum_{j=1}^{n} \left(
\dfrac{1}{n} \sum_{k=1}^{n} \log a_{ik}- \dfrac{1}{n}
\sum_{k=1}^{n} \log a_{jk} \right)^{2}}{ \dsum_{i=1}^{n}
\sum_{j=1}^{n} (\log a_{ij})^{2}} \, .
\end{equation}
for all matrices $\mathbf{A} \neq (1)_{n \times n}$, and zero if $\mathbf{A} = (1)_{n \times n}$.
\end{definition}
\citet{PelaezLamata2003} defined an inconsistency index
for a pairwise comparison matrix as the average of all the determinants of its
$3 \times 3$ submatrices, each containing a different
transitivity of the original matrix.
\begin{definition}[Index $CI^{\ast}$ \citep{PelaezLamata2003}]
\label{def:CI*}
Given a pairwise comparison matrix of order $n$, the index $CI^{\ast}$ is
\begin{equation}
\label{svil_det_3X3} CI^{*}(\mathbf{A}) = \sum_{i=1}^{n-2}
\sum_{j=i+1}^{n-1} \sum_{k=j+1}^{n} \left(
\frac{a_{ik}}{a_{ij}a_{jk}} + \frac{a_{ij}a_{jk}}{a_{ik}} - 2
\right) \bigg/ \binom{n}{3} \: .
\end{equation}
\end{definition}
\citet{ShiraishiEtAl1998} proposed the
coefficient $c_{3}$ of the characteristic polynomial of
$\mathbf{A}$ as an index of inconsistency. \citet{BrunelliCritchFedrizzi2011} proved
that index $CI^{*}$ is proportional to $c_3$ .

\citet{SteinMizzi2007} considered the general result
that the columns of a consistent pairwise comparison matrix are
proportional, i.e. $\text{rank}(\mathbf{A})=1$, if and only if
$\mathbf{A}$ is consistent. Thus, they formulated an index which
takes into account how far the columns are from being proportional
to each other.
\begin{definition}[Harmonic Consistency Index \citep{SteinMizzi2007}]
\label{def:HCI}
Let $\mathbf{A}$ be a pairwise comparison matrix and
 $s_j=\sum_{i=1}^n a_{ij}$ for $j=1, \ldots ,n$.
Then, the harmonic consistency index is
\begin{equation}
\label{eq:HCI} HCI(\mathbf{A})=\frac{(HM(\mathbf{A})-n)(n+1)}{n(n-1)},
\end{equation}
where $HM(\mathbf{A})$ is the harmonic mean of $(s_1,\ldots,s_n)$:
\begin{equation}
\label{eq:mediaarmonica}
HM(\mathbf{A})=\frac{n}{\sum_{j=1}^{n}\frac{1}{s_{j}}}.
\end{equation}
\end{definition}

\citet{Koczkodaj1993} and \citet{DuszakKoczkodaj1994} introduced a max-min based inconsistency index which was later compared with $CI$ by \citet{BozokiRapcsak2008}.
\citet{CavalloD'Apuzzo2009} characterized pairwise comparison matrices
by means of Abelian linearly ordered groups and stated their inconsistency index in this general framework.

Another index, $NI^{\sigma}_{n}$, was introduced by \citet{RamikKorviny2010} to estimate the inconsistency of
pairwise comparison matrices with elements expressed as triangular fuzzy numbers. Expressing judgments in such a
way is popular to account for uncertainties in
the decision making process. Nevertheless, pairwise comparison matrices can be seen as special cases of matrices with fuzzy entries and therefore this index can be introduced in the context of pairwise comparison matrices with real entries.
\begin{definition}[Index $NI^{\sigma}_{n}$ \citep{RamikKorviny2010}]
\label{def:NI}
Given a real number $\sigma > 0$ and a pairwise comparison matrix $\mathbf{A}\in\mathcal{A}$ of order $n$ with entries in the interval $[1/\sigma,\sigma]$, the index $NI^{\sigma}_{n}$ is defined as
\[
NI_{n}^{\sigma}(\mathbf{A})= \gamma_{n}^{\sigma} \max_{i,j} \left\{ \left| \frac{w_{i}}{w_{j}}-a_{ij} \right| \right\},
\]
where the weights are obtained by means of the geometric mean method (\ref{eq:mediageometrica}) and
\[
\gamma_{n}^{\sigma} =
\begin{cases}
\frac{1}{\max \left\{ \sigma - \sigma^{\frac{2-2n}{n}}, \sigma^{2}\left(\left(\frac{2}{n}\right)^{\frac{2}{n-2}} - \left( \frac{2}{n} \right)^{\frac{n}{n-2}} \right) \right\}}, & \text{if }\sigma < \left( \frac{n}{2} \right) ^{\frac{n}{n-2}} \\
\frac{1}{\max \left\{ \sigma- \sigma^{\frac{2-2n}{n}}, \sigma^{\frac{2n-2}{n}} - \sigma \right\} }, & \text{if }\sigma \geq \left( \frac{n}{2} \right)^{\frac{n}{n-2}}
\end{cases}
\]
is a positive normalization factor.
\end{definition}
Other notable indices are the parametric method by \citet{OseiBryson2006} and the ambiguity index by \citet{Salo1993}.

\section{Axioms}
\label{sec:properties}
In spite of the large number of indices, the question on how well
they estimate inconsistency of pairwise comparisons has been left
unanswered. To answer this question, in this section we introduce
and justify five properties to narrow the general definition of
inconsistency index given in (\ref{eq:inconsistency_index}) and to
shed light on those indices which do not satisfy minimal reasonable
requirements. Throughout this and the next sections we are going to
propose some examples in order to provide numerical
and visual evidence of the \emph{necessity} of the following
axiomatic system.

\subsection*{Axiom 1: Existence of a unique element representing consistency}
With axiom 1 (A1) we require that all the consistent matrices are
identified by a unique real value of an inconsistency index. This
allows to distinguish between matrices that either belong or do not to
$\mathcal{A}^{*}$. Formally, A1 is as follows.
\begin{property}
An inconsistency index $I$ satisfies A1, if and only if
 \begin{equation}
 \label{A1}
\exists ! \nu \in \mathbb{R} \text{ such that } I(\mathbf{A})= \nu \Leftrightarrow \mathbf{A} \in \mathcal{A}^{*}
 \end{equation}
 \end{property}

\begin{example}
The following inconsistency index satisfies A1 with $\nu=0$.
\[
I(\mathbf{A})= \sum_{i=1}^{n}\sum_{j=1}^{n}\sum_{k=1}^{n} | a_{ik}-a_{ij}a_{jk} | \, .
\]
\end{example}


%


For sake of simplicity, we assume, without loss of generality,
that, for every inconsistency  index $I(\mathbf{A})$, the value
$\nu$ associated with each consistent matrix is the minimum value
of the index: $I(\mathbf{A})\geq \nu ~\forall
\mathbf{A}\in\mathcal{A}$. The assumption is that the more
inconsistent is $\mathbf{A}$, the greater is $I(\mathbf{A})$. Some
already introduced indices assume the opposite. By considering,
for example, the index introduced by \citet{ShiraishiEtAl1998}, it
is $c_{3}(\mathbf{A})\leq 0 ~\forall \mathbf{A}\in\mathcal{A}$,
while the consistency value is $c_{3}(\mathbf{A})=\nu= 0 ~\forall
\mathbf{A}\in\mathcal{A}^{*}$. Nevertheless, in such cases it is
sufficient to change the sign of the index to fulfill our
assumption.

\subsection*{Axiom 2: Invariance under permutation of alternatives}
It is desirable that an inconsistency index does not depend on the order in which the
alternatives are associated with rows and columns of
$\mathbf{A}$. Therefore, an inconsistency index should be invariant under
row-column permutations. To formalize this second axiom (A2), we recall that a permutation matrix
is a square binary matrix
$\mathbf{P}$ that has exactly one entry equal to 1 on each row and
each column and 0's elsewhere (see \citealt{HornJohnson1985}). We also recall that
$\mathbf{P}\mathbf{A}\mathbf{P}^{\mathrm{T}}$ is the matrix obtained from
$\mathbf{A}$ through the row-column permutations associated with $\mathbf{P}$.
\begin{property}
An inconsistency index $I$ satisfies A2, if and only if
\begin{equation}
\label{eq:permutation}
I \left( \mathbf{P}\mathbf{A}\mathbf{P}^{\mathrm{T}} \right)=I(\mathbf{A}) \; \; \forall \mathbf{A}\in \mathcal{A}
\end{equation}
and for any permutation matrix $\mathbf{P}$.
\end{property}
%
\begin{example}
Given a pairwise comparison matrix $\mathbf{A}$ and a permutation matrix $\mathbf{P}$
\[
\mathbf{A}=
\begin{pmatrix}
1   & 2   & 5 \\
1/2 & 1   & 2 \\
1/5 & 1/2 & 1
\end{pmatrix} ~~~~
\mathbf{P}=
\begin{pmatrix}
0   & 1   & 0 \\
1   & 0   & 0 \\
0   & 0   & 1
\end{pmatrix}
\]
one obtains
\[
\mathbf{P}\mathbf{A}\mathbf{P}^{\mathrm{T}}=
\begin{pmatrix}
0   & 1   & 0 \\
1   & 0   & 0 \\
0   & 0   & 1
\end{pmatrix}
\begin{pmatrix}
1   & 2   & 5 \\
1/2 & 1   & 2 \\
1/5 & 1/2 & 1
\end{pmatrix}
\begin{pmatrix}
0   & 1   & 0 \\
1   & 0   & 0 \\
0   & 0   & 1
\end{pmatrix}
=
\begin{pmatrix}
1   & 1/2   & 2 \\
2   & 1   & 5 \\
1/2   & 1/5   & 1
\end{pmatrix}
\]
for which (\ref{eq:permutation}) is required to hold.
\end{example}

\subsection*{Axiom 3: Monotonicity under reciprocity-preserving mapping}
Unlike the previous axioms, which were simple regularity conditions
imposed to $I(\mathbf{A})$, axiom 3 (A3) is more constraining. The
idea is that, if preferences are intensified, then an inconsistency
index cannot return a lower value. However, before we formalize it,
we describe its meaning. If all the expressed preferences indicate
indifference between alternatives, it is $a_{ij}=1 \;  \forall i,j$,
and $\mathbf{A}$ is consistent. Going farther from this uniformity
means having stronger judgments and this should \emph{not} make
their possible inconsistency \emph{less} evident. In other words,
intensifying the preferences (pushing them away from indifference)
should not de-emphasize the characteristics of these preferences and
their possible contradictions. Clearly, the crucial point is to find
a transformation which can intensify preferences and preserve their
structure at the same time. In the following, we are going to prove
that such a transformation exists and is unique.
Given $\mathbf{A}=(a_{ij}) \in \mathcal{A}$, we denote such transformation with $\hat{a}_{ij} = f(a_{ij})$. The newly constructed matrix $\hat{\mathbf{A}}=(\hat{a}_{ij}) $ obtained from $\mathbf{A}$ by means of $f$ must be positive and reciprocal so that it still belongs to $\mathcal{A}$. Hence
\[
\hat{a}_{ji} = 1 / \hat{a}_{ij},
\]
\noindent which is
\begin{align*}
f(a_{ji})&= 1 / f(a_{ij}) \\
f(1/a_{ij}) &= 1 / f(a_{ij}) \\
f(a_{ij}) f(1/a_{ij})  &= 1
\end{align*}
\noindent or, more compactly, with $a_{ij}=x$,
\begin{equation}
\label{Cauchy1}
 f(x) f(1/x) = 1 .
\end{equation}
\noindent Equation (\ref{Cauchy1}) is a special case, for $y=1/x$,
of the well-known Cauchy functional equation
\begin{equation}
\label{Cauchy2}
 f(x) f(y) = f(x y) .
\end{equation}
In fact, by substituting $x=1$ into (\ref{Cauchy1}), it is
$f(1)f(1)=1$. Since $f$ must be positive, it follows $f(1)=1$. Then,
(\ref{Cauchy1}) can also be written in the form
$f(x)f(1/x)=f(x(1/x))=f(1)$.
Taking into account that $x= a_{ij} > 0$, it is therefore
sufficient to assume the continuity of $f$ in order to obtain a
unique non-trivial solution of (\ref{Cauchy1}) (see \citealt{Aczel1966})
\begin{equation}
\label{soluzione-Cauchy}
 f(x) = x^b , \; \; b \in \mathbb{R} .
\end{equation}
Therefore, the only continuous transformation $f(a_{ij})$ preserving
reciprocity is (\ref{soluzione-Cauchy}), i.e. $f(a_{ij})= a_{ij}^b$.
In the following, we will denote matrix $(a_{ij}^{b})_{n \times n}$
as $\mathbf{A}(b)$. Clearly, for $b>1$ each entry $a_{ij} \neq 1 $
is moved farther from indifference value 1, which represents an
intensification of preferences:
\begin{align*}
 b>1 , \;\;  a_{ij} >1    ~&\Rightarrow a_{ij}^b > a_{ij}>1   \\
 b>1 , \;\;  0< a_{ij} <1 ~&\Rightarrow 0<a_{ij}^b < a_{ij}<1 \, .
\end{align*}
\noindent The opposite occurs for $0<b<1$, thus representing a
weakening of the preferences. For $b=0$ full indifference is
obtained, $a_{ij}^b = 1$, while $b<0$ corresponds to preference
reversal. Moreover, transformation (\ref{soluzione-Cauchy}) is
consistency-preserving, i.e. if $\mathbf{A}=(a_{ij}) $ is
consistent, then also $\mathbf{A}(b)=(a_{ij}^{b})$ is
consistent. The proof is straightforward, since from $a_{ij}a_{jk}
 =  a_{ik} $ immediately follows $a_{ij}^b a_{jk}^b  =
a_{ik}^b$. Furthermore,
(\ref{soluzione-Cauchy}) is also the \emph{unique} consistency-preserving
transformation, the proof being similar to the one described above
for reciprocity.

To summarize, the only continuous transformation that intensifies
preferences and preserves reciprocity (and consistency) is
$f(a_{ij})= a_{ij}^b$ with $b>1$ and then A3 can be formalized.
\begin{property}
Define $\mathbf{A}(b) = \left( a_{ij}^{b} \right)_{n \times n}$. Then, an inconsistency index $I$ satisfies A3 if and only if
\begin{equation}
\label{A3} I ( \mathbf{A}(b) ) \geq I( \mathbf{A} ) ~~\forall b>1 ,
\; \; \forall \mathbf{A}\in \mathcal{A} .
\end{equation}
\end{property}

\begin{example}
\label{exampleA3}
Consider the following matrix
\[
\mathbf{A}=
\begin{pmatrix}
1   & 2   & 1/2 \\
1/2 & 1   & 2 \\
2   & 1/2 & 1
\end{pmatrix}.
\]
Then, modifying entries of
$\mathbf{A}$ by  means of function $f$ with exponent $b=3$ one
obtains the following matrix
\[
\mathbf{A}(3)=
\begin{pmatrix}
1^{3}   & 2^{3}   & 1/2^{3} \\
1/2^{3} & 1^{3}   & 2^{3} \\
2^{3}   & 1/2^{3} & 1^{3}
\end{pmatrix} =
\begin{pmatrix}
1   & 8   & 1/8 \\
1/8 & 1   & 8 \\
8   & 1/8 & 1
\end{pmatrix}.
\]
If an inconsistency index $I$ satisfies A3, then it must be $I(\mathbf{A}(3)) \geq
I(\mathbf{A}) $. In words, if A3 holds, then $\mathbf{A}(3)$ cannot be judged less inconsistent than
$\mathbf{A}$.
\end{example}

Note that transformation $f(a_{ij})=a_{ij}^{b}$ has been used for
other scopes. \citet{Saaty1977} himself proposed it in his
seminal paper to show that his results on consistency were general
enough to cover scales other than $[1/9,9]$. Such a function was
also employed by \citet{Herrera-Viedma2004} to
find a suitable mapping to rescale the entries of a pairwise
comparison matrix into the interval $[1/9,9]$ and by \citet{FedrizziBrunelli2009} to define
consistency-equivalence classes.

\subsection*{Axiom 4: Monotonicity on single comparisons}
\label{sec:A4} Let us consider a consistent matrix with \emph{at
least} one non-diagonal entry  $a_{pq} \neq 1$. If we increase or
decrease the value of $a_{pq}$, and modify its reciprocal $a_{qp}$
accordingly, then the resulting matrix is not anymore consistent. In
fact, in agreement with A1, the resulting matrix will have a degree
of inconsistency which exceeds that of the consistent matrix. Axiom
4 (A4) establishes a condition of monotonicity for the inconsistency
index with respect to single comparisons by requiring that the
larger the change of $a_{pq}$ from its consistent value, the more
inconsistent the resulting matrix will be. More formally, given a
consistent matrix \mbox{$\mathbf{A} \in \mathcal{A}^{*}$}, let
$\mathbf{A}_{pq}(\delta)$ be the inconsistent matrix obtained from
\textbf{A} by replacing the entry $a_{pq}$ with $a_{pq}^{\delta}$,
where $\delta \neq 1$. Necessarily, $a_{qp}$ must be replaced by
$a_{qp}^{\delta}$ to preserve reciprocity. Let
$\mathbf{A}_{pq}(\delta')$ be the inconsistent matrix obtained from
\textbf{A} by replacing entries $a_{pq}$ and $a_{qp}$ with
$a_{pq}^{\delta'}$ and $a_{qp}^{\delta'}$ respectively. A4 can then
be formulated as

\begin{equation}
\label{monotonicity}
\begin{split}
 \delta' > \delta > 1 & \Rightarrow I(\mathbf{A}_{pq}(\delta')) \geq I(\mathbf{A}_{pq}(\delta)) \\
 \delta' < \delta < 1 & \Rightarrow I(\mathbf{A}_{pq}(\delta')) \geq I(\mathbf{A}_{pq}(\delta)) \, .
\end{split}
\end{equation}
Axiom 4 can be equivalently formalized as follows.
\begin{property}
An inconsistency index $I$ satisfies A4, if and only if
$I(\mathbf{A}_{pq}(\delta))$ is a non-decreasing function of
$\delta$ for $\delta > 1$ and a non-increasing function of
$\delta$ for $\delta < 1$, for all the $\mathbf{A}\in \mathcal{A}^{*}$ and $p,q=1,\ldots,n$.
\end{property}

\begin{example}
\label{exampleA4}
Consider the consistent matrix
\[
\mathbf{A}=
\begin{pmatrix}
1     & 2   & 4 \\
1/2   & 1   & 2 \\
1/4   & 1/2 & 1
\end{pmatrix} \in \mathcal{A}^{*}.
\]
Then, choosing, for instance,
entry $a_{13}$ and changing its value and the value of its
reciprocal accordingly, we obtain
\begin{align}
\mathbf{A}'=
\begin{pmatrix}
1   & 2   & 5 \\
1/2 & 1   & 2 \\
1/5   & 1/2 & 1
\end{pmatrix}~~~
\mathbf{A}''=
\begin{pmatrix}
1   & 2   & 9 \\
1/2 & 1   & 2 \\
1/9 & 1/2 & 1
\end{pmatrix}
\end{align}
If an inconsistency index $I$ satisfies A4, then $I(\mathbf{A}'') \geq I(\mathbf{A}') \geq I(\mathbf{A})$, where the inequality between $I(\mathbf{A})$ and $I(\mathbf{A}')$ becomes strict if A1 holds. Note that, in this example, $\mathbf{A}'=\mathbf{A}_{13}(\delta)$ with $\delta = \log_{4} 5$ and $\mathbf{A}''=\mathbf{A}_{13}(\delta ')$ with $\delta ' = \log_{4} 9$.
\end{example}

Moreover, we note that A4 formalizes a property proved by \citet{AupetitGenest1993} for Saaty's Consistency Index and considered by the authors as a necessary property. Furthermore, the case of a potentially consistent matrix with one deviating comparison was considered by \citet{Bryson1995} in a property that he called `single outlier neutralization' and by \citet{ChooWedley2004} in their comparative study of methods
to elicit the weight vector. A4 is also in the spirit of other known axiomatic systems. As examples, \citet{CookKress1988} considered two matrices differing by only one comparison and \citet{KemenySnell1962} proposed a similar axiomatic assumption for the distance between rankings.

\subsection*{Axiom 5: Continuity}

As defined in (\ref{eq:inconsistency_index}), an inconsistency index
$I( \mathbf{A} )$ is a function of $\mathbf{A} \in \mathcal{A}$.
With this fifth axiom (A5), the continuity of the function is
required in the set $ \mathcal{A}$. More precisely, an index $I(
\mathbf{A} )$ is considered as a function of the ${n \choose 2}$
variables $a_{ij}, i<j$ and continuity of $I( \mathbf{A} )$ is meant
as the continuity of a function of $n \choose 2$ real variables.
Axiom 5 can be formalized as follows
\begin{property}
An inconsistency index $I( \mathbf{A} )$ satisfies A5 if and only if
it is a continuous function of the entries $a_{ij}$ of $\mathbf{A}$,
with $a_{ij}>0, a_{ij}a_{ji}=1 ~\forall i,j$.
\end{property}
The importance of continuity in mathematical modelling has origins
in the  fact that it guarantees that infinitesimal variations in the
input only generates an infinitesimal variation of the output, thus
excluding functions with `jumps'.

\subsection{Significance of the axioms}

Let us briefly discuss the \emph{necessity} of the five axioms by showing that their violation could result in an unreasonable inconsistency measurement:

\begin{itemize}
    \item If A1 is violated, two perfectly consistent matrices can have two different numerical consistency evaluations.

    \item If A2 is violated, different consistency evaluations could be associated to the same set of preferences, simply by renaming of alternatives.

    \item The effect of violation of A3 is apparent from Example \ref{exampleA3}. If A3 is not respected, matrix $\mathbf{A}(3)$, where inconsistent preferences are reinforced, could be evaluated less inconsistent than $\mathbf{A}$.

    \item Let us consider Example \ref{exampleA4} to show the necessity of A4. In the Example, matrix $\mathbf{A}''$
    clearly differs from the consistent matrix $\mathbf{A}$ more than $\mathbf{A}'$ does. As a consequence, $\mathbf{A}''$
    cannot be evaluated less inconsistent than $\mathbf{A}'$.

    \item As stated by \citet{Barzilai1998}, continuity `is a reasonable requirement of any measure of amount of
    inconsistency'. In the proof of Proposition \ref{prop:RE}, referring to Barzilai's own index,
    we will show that a discontinuous index may assign the largest
    inconsistency evaluation to a matrix which is arbitrarily close
    to a consistent one.
\end{itemize}
%
%
In spite of the reasonability of A1--A5, they could be suspected of
being too weak in order
 to characterize an inconsistency index. On the contrary, they turn out to be strictly demanding,
 since Propositions \ref{prop:RE}, \ref{prop:NI}, \ref{prop:HCI}, and \ref{prop:GW}
 will surprisingly show that they are not satisfied by four indices based on seemingly reasonable definitions.

\subsection{Logical consistency and independence}
\label{Independence}
A natural question is whether the axiomatic properties A1--A5 form
an axiomatic system or not. In fact, in an axiomatic system, the
axioms must be consistent (in a logical sense) and independent. The
existence of, at least, one index satisfying A1--A5
proves that the axiomatic system is not logically contradictory and
therefore the system is \emph{logically consistent}. Another
important result regards the \emph{independence} of the axioms.
Proving  the independence would show that the axioms are not
redundant, and therefore all of them shall be considered necessary.

\begin{theorem}
\label{th:independency} Axiomatic properties A1--A5 are logically
consistent and independent.
\end{theorem}
\begin{proof}
Propositions \ref{prop:Saaty}, \ref{prop:CI^{*}} and \ref{prop:GCI} in Section \ref{sec:results}
state that indices $CI$, $CI^{*}$ and $GCI$, respectively, satisfy all
the five properties. Then, the axiomatic properties A1--A5 are
logically consistent.
Independence of a given axiom can be shown by providing an example
of index satisfying all axioms except the one at stake. To prove the
independence of A1, one can consider the following \textit{ad hoc}
constructed index,
\[
I_{1}(\mathbf{A})=\max \left\{ CI^{*}(\mathbf{A})-1 , 0 \right\} \,
.
\]
\noindent As a consequence of Proposition \ref{prop:CI^{*}}, index
$I_{1}$ satisfies A2--A5. Nevertheless, $I_{1}$ assigns value $0$
also to some inconsistent pairwise comparison matrices, so that it
does not satisfy A1.
To prove independence of A2, one could instead consider
\[
I_{2}(\mathbf{A}) = \sum_{i=1}^{n-2} \sum_{j=i+1}^{n-1}
\sum_{k=j+1}^{n}  \left( \frac{a_{ik}}{a_{ij}a_{jk}} +
\frac{a_{ij}a_{jk}}{a_{ik}} - 2 \right) w_{ijk}
\]
with $w_{ijk}>0 ~\forall \; i,j,k, \;\; 0<i<j<k \leq n$ and $w_{ijk}
\neq w_{i'j'k'}$  for some $0<i<j<k \leq n$ and $ 0<i'<j'<k' \leq
n$. Index $I_{2}$ does not satisfy A2, due to the presence of
$w_{ijk}$. Conversely, it satisfies all the other axioms, the proof
being similar to that of Proposition \ref{prop:CI^{*}}.
\noindent Independence of A3 directly follows from Proposition
\ref{prop:HCI}.
\noindent To prove independence of A4, we propose the following
index,
\[
I_{4}(\mathbf{A})=\left( \max_{i \neq j} \max \{ a_{ij},a_{ji} \} -
\min_{i \neq j} \max \{ a_{ij},a_{ji} \} + \epsilon \right) \times
(CI^{*}(\mathbf{A}))^{0.1} \, .
\]
\noindent It can be proved that $I_{4}$ fails to satisfy A4 for a
convenient choice of $\epsilon > 0$ and $\mathbf{A}$.
Finally, it is easy to check that
\[
I_{5}(\mathbf{A})=
\begin{cases}
\text{0} & \text{if $\mathbf{A}\in \mathcal{A}^{*}$},\\
\text{1} & \text{if $\mathbf{A}\notin \mathcal{A}^{*}$}
\end{cases}
\]
satisfies axioms A1--A4, but is not continuous and therefore does not fulfill A5,
thus showing its independence.
\end{proof}

Next, we shall investigate if existing inconsistency
indices---especially those defined in the previous section---satisfy
the axioms A1--A5.

\section{On the satisfaction of the axioms}
\label{sec:results}

We first consider three inconsistency indices and prove that they
satisfy all the axioms. Only later, we shall prove that some others
do not satisfy some axioms. As anticipated in the introduction, most of the
proofs are given in the appendix in order to simplify the
description.

Saaty's $CI$ and indices $CI^{*}$ and $GCI$ satisfy the five axioms
A1--A5. We can formalize it in the following propositions.

\begin{proposition}
\label{prop:Saaty} Saaty's Consistency Index $CI$ (\ref{eq:CI})
satisfies the five axioms A1--A5.
\end{proposition}
%
%
\begin{proposition}
\label{prop:CI^{*}}
Index $CI^{*}$ satisfies the five axioms A1--A5.
\end{proposition}
%
%
\begin{proposition}
\label{prop:GCI}
The Geometric Consistency Index $GCI$ satisfies the five axioms
A1--A5.
\end{proposition}
%

Let us now consider Barzilai's inconsistency index $RE$ and
formulate the following interesting result,
\begin{proposition}
\label{prop:RE_inv} Let $\mathbf{A}\in \mathcal{A}$ and $
\mathbf{A}(b) = (a_{ij}^{b})_{n \times n} $. Then
$RE(\mathbf{A})=RE(\mathbf{A}(b))~\forall b \neq 0$ and therefore $RE$ is
invariant w.r.t. $f(a_{ij})=a_{ij}^{b}~\forall b>0$.
\end{proposition}
%
%
\noindent Proposition \ref{prop:RE_inv} could be seen as a
restriction of A3, where an inconsistency index is required to be
\emph{invariant} under function $f(a_{ij})=a_{ij}^{b}$. Clearly,
this implies that index $RE$ satisfies A3. The general result on
index $RE$ is stated by the following proposition
\begin{proposition}
\label{prop:RE}
Index $RE$ satisfies A1--A3, but it does not satisfy A4 and A5.
\end{proposition}
%
%
%

\noindent Note that Proposition \ref{prop:RE} disproves the
continuity of $RE$ claimed in the original paper by \citet{Barzilai1998}.

The following proposition concerns index $NI_{n}^{\sigma}$,
introduced by \citet{RamikKorviny2010}, see
definition \ref{def:NI}. It remains unproved whether
$NI_{n}^{\sigma}$ satisfies A3 or not.

\begin{proposition}
\label{prop:NI} Index $NI_{n}^{\sigma}$ satisfies axioms A1, A2 and
A5 but it does not satisfy A4.
\end{proposition}
\begin{proof}
 The proof that $NI_{n}^{\sigma}$ satisfies axioms A1 and A2 is straightforward.
 To prove that $NI_{n}^{\sigma}$ does not satisfy A4, let us consider the following
consistent pairwise comparison matrix,
\begin{equation}
\label{matr:NI} \mathbf{A}=
\begin{pmatrix}
1 & 1/3 & 1/3 & 1/9 \\
3 & 1 & 1 & 1/3 \\
3 & 1 & 1 & 1/3 \\
9 & 3 & 3 & 1 \\
\end{pmatrix}\in \mathcal{A}^{*}.
\end{equation}
If entry $a_{14}$ is changed, and its reciprocal $a_{41}$ varies
accordingly, then, the violation of A4 can be appreciated in Figure
\ref{fig:NI}, where $NI_{n}^{\sigma}(\mathbf{A})$ is plotted as a
function of $a_{14}$, being $n=4$ and $\sigma =9 $.
\begin{figure}[htbp]
    \centering
        \includegraphics[width=0.45\textwidth]{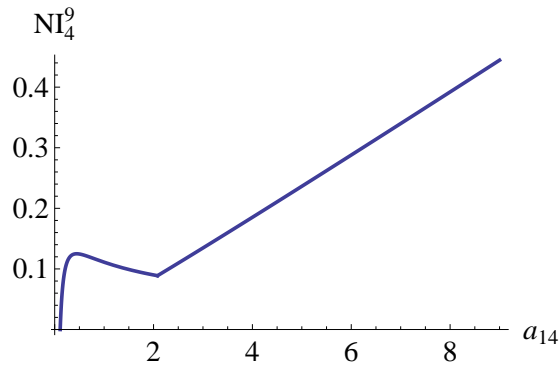}
    \caption{$NI_{4}^{9}(\mathbf{A})$ as a function of $a_{14}$}
    \label{fig:NI}
\end{figure}
In fact, A4 implies that such a function should be monotonically
increasing for $a_{14} > 1/9$, but this is not the case in this
example, since, e.g., $2 > 0.5$ but the value of
$NI_{4}^{9}(\mathbf{A})$ corresponding to $a_{14} = 2$ is smaller
than the value of $NI_{4}^{9}(\mathbf{A})$ corresponding to $a_{14}
= 0.5$.
Continuity of $NI_{n}^{\sigma}$ follows from continuity of
$\max\{\cdot , \cdot \}$, so that A5 is satisfied.
\end{proof}
\noindent We can only conjecture that the behavior described in Figure \ref{fig:NI}
is related with the fact, noted by \citet{Brunelli2011},
that index $NI_{n}^{\sigma}$ fails to identify the most
inconsistent $3 \times 3$ matrix.

We next consider another index which fails to satisfy one of the
axioms: the Harmonic Consistency Index (\ref{eq:HCI}).
\begin{proposition}
\label{prop:HCI}
Index $HCI$ satisfies A1, A2, A4 and A5 but it does not satisfy A3.
\end{proposition}
%
%
\noindent The following example derives from the proof of
proposition \ref{prop:HCI} and is aimed to clarify the behavior of
$HCI$ and to show the importance of A3.

\begin{example} %
\label{ex:HCI} %
Consider the following matrix $\mathbf{A}$ and its derived matrix $\mathbf{A}(b)=\left( a_{ij}^{b} \right)$
\[
\mathbf{A}=
\begin{pmatrix}
1   & 4 & 1/2  & 2 \\
1/4 & 1   & 1/4 & 2 \\
2   & 4 & 1 & 2 \\
1/2 & 1/2 & 1/2 & 1
\end{pmatrix}~~~~~
\mathbf{A}(b)=
\begin{pmatrix}
1^{b}   & 4^{b} & (1/2)^{b}   & 2^{b} \\
(1/4)^{b} & 1^{b}  & (1/4)^{b} & 2^{b} \\
2^{b}   & 4^{b} & 1^{b} & 2^{b} \\
(1/2)^{b} & (1/2)^{b} & (1/2)^{b} & 1
\end{pmatrix}.
\]
\noindent It is possible to illustrate the behavior of
$HCI(\mathbf{A}(b))$ by means of Figure \ref{fig:HCI}. The index
initially increases, but then it decreases and converges to full
$HCI$-consistency as $b$ grows.

\begin{figure}[htbp]
    \centering
        \includegraphics[width=0.45\textwidth]{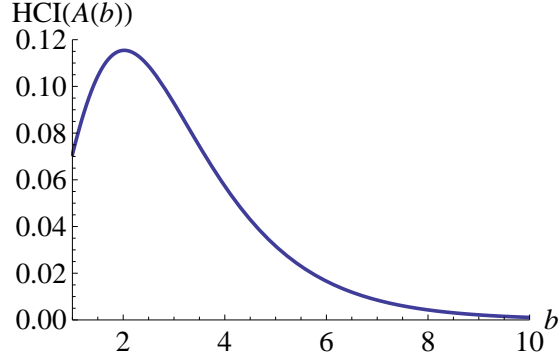}
    \caption{Index $HCI(\mathbf{A}(b))$ as a function of $b$}
    \label{fig:HCI}
\end{figure}
\end{example}

The last index considered in this section is the index $GW$ of
\citet{GoldenWang1989} and the following
proposition states the corresponding results. It remains unproved
whether $GW$ satisfies A4 or not.


\begin{proposition}
\label{prop:GW} Index $GW$ satisfies A1, A2 and A5. If the priority
vector is computed by means of the geometric mean method, then index
$GW$ does not satisfy A3.
\end{proposition}
%
%
The following example derives from the proof of Proposition
\ref{prop:GW} and, similarly to example \ref{ex:HCI}, shows the
convergence to zero of the index $GW$.
\begin{example}
\label{ex:GW}%
Consider the pairwise comparison matrix
$\mathbf{A}(b)$ obtained from $\mathbf{A}$ as
\begin{equation}
\label{A_GW}
\mathbf{A}=
\begin{pmatrix}
1   & 3 & 1/4 & 2 \\
1/3   & 1 & 1/7  & 2 \\
4 & 7 & 1 & 6 \\
1/2 & 1/2 & 1/6 & 1
\end{pmatrix} ~~~~
\mathbf{A}(b)=
\begin{pmatrix}
1^{b}  & 3^{b} & (1/4)^{b} & 2^{b} \\
(1/3)^{b}   & 1^{b} & (1/7)^{b} & 2^{b} \\
4^{b} & 7^{b} & 1^{b} & 6^{b} \\
(1/2)^{b} & (1/2)^{b} & (1/6)^{b} & 1^{b}
\end{pmatrix}
\end{equation}
and note that the third row of $\mathbf{A}$ contains all the greatest
elements of each column. It is possible to plot the behavior of $GW$
and obtain the graph in Figure \ref{fig:GW}, which represents the
$GW$-inconsistency of matrix $\mathbf{A}(b)$ in (\ref{A_GW}).
\begin{figure}[htbp]
    \centering
        \includegraphics[width=0.45\textwidth]{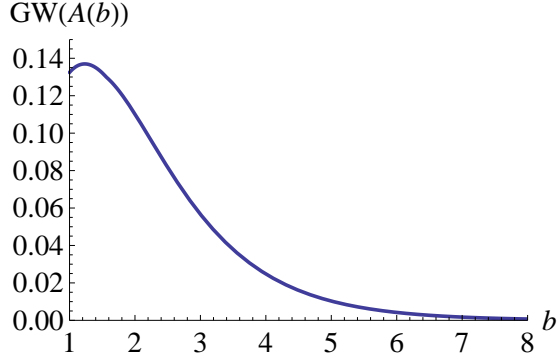}
    \caption{Index $GW(\mathbf{A}(b))$ as a function of $b$}
    \label{fig:GW}
\end{figure}
\end{example}
Finally, Table \ref{tab:summary} summarizes the findings obtained in this section.
\begin{table}[htbp]
    \small
    \centering
        \begin{tabular}{l|ccccc}
            ~                                   & A1 & A2 & A3 & A4 & A5 \\
            \hline
      $CI$ (def. \ref{def:CI})                  &  Y & Y  & Y  & Y  &  Y \\
      $GW$ (def. \ref{def:GW})                  &  Y & Y  & N  & ?  &  Y \\
      $GCI$ (def. \ref{def:GCI})                &  Y & Y  & Y  & Y  &  Y \\
      $RE$ (def. \ref{def:RE})                  &  Y & Y  & Y  & N  &  N \\
      $CI^{*}$ (def. \ref{def:CI*})             &  Y & Y  & Y  & Y  &  Y \\
      $HCI$ (def. \ref{def:HCI})                &  Y & Y  & N  & Y  &  Y \\
      $NI^{\sigma}_{n}$ (def. \ref{def:NI})     &  Y & Y  & ?  & N  &  Y

        \end{tabular}
    \caption{Summary of propositions: Y= axiom is satisfied, N= axiom is not satisfied, ?= unknown}
    \label{tab:summary}
\end{table}
\section{Discussion and Future Research}
\label{sec:open_question}
Propositions \ref{prop:NI}, \ref{prop:HCI} and \ref{prop:GW},
together with the corresponding examples, suggest that A3 and A4 are
the most demanding axioms. Let us make some other
remarks to clarify A3 and A4.
First, we propose a geometrical interpretation that could be useful to
emphasize the role of A4 in requiring the non-decreasing property of
an inconsistency index when moving away from consistency. Let us
represent a consistent matrix $\mathbf{A} \in \mathcal{A}^{*}$ as a
point in the Cartesian space $\mathbb{R}^{n(n-1)/2}$, where the
dimension ${n(n-1)/2}$ is the number of upper-diagonal elements
which are necessary and sufficient to identify a pairwise comparison
matrix of order $n$. By increasing (decreasing) entry $a_{pq}$, the
point $\mathbf{A}$ departs from set $\mathcal{A}^{*}$ and moves in
the direction of the corresponding axis. Thus, A4 requires that an
inconsistency index does not decrease whenever $\mathbf{A}$ moves
away from the initial consistent position, in
any of the ${n(n-1)/2}$ possible directions. 
On the other hand, by referring to the same geometrical
representation in the Cartesian space, the type of translation of
point $\mathbf{A}$ induced by A3 is different from the
one induced by A4, so that the joint effect is more general than
the single ones.\\

In decision problems based on pairwise comparisons there are two
phases: preference elicitation and priority vector computation. In
the previous sections we defined and studied five axioms
characterizing the inconsistency evaluation of the preferences
elicited by a decision maker independently from the method used in
deriving the priority vector. Therefore, we focused on pairwise
comparison matrix $\mathbf{A}$ and the property of transitivity
$a_{ik}=a_{ij}a_{jk}~\forall i,j,k$. Nevertheless, by considering that
consistency of a pairwise comparison matrix can equivalently be
characterized by property $a_{ij}=w_{i}/w_{j}~\forall i,j$, it is possible and
relevant to study also the relationship between the inconsistency
indices and the methods used for priority vector computation.
Other investigations will
add new insight on the relationship between the inconsistency
indices and the methods used for computing priority vectors.


\subsection{Conclusions}
The purpose of this paper was to introduce some formal order in the
topic of consistency evaluation for pairwise comparison matrices.
We proposed few and simply justifiable axiomatic properties to characterize
inconsistency indices, discovering that some indices proposed in the literature fail to satisfy some
properties. We hope that our proposal will open a debate and
stimulate further studies.



\appendix

\section*{Appendix}

\begin{proof}[Proof of Proposition \ref{prop:Saaty}] In the following we prove each axiom separately.\\
\begin{description}
    \item[A1] This was proved already by \citet{Saaty1977}.
    \item[A2] It is known that the  characteristic polynomial of a matrix $\mathbf{A}$ equals
    the characteristic polynomial of the matrix $\mathbf{S}^{-1} \mathbf{A} \mathbf{S}$ where $\mathbf{S}$
    is any non-singular matrix and $\mathbf{S}^{-1}$ its inverse (see \citet{HornJohnson1985}, p. 45).
    We also know that any permutation matrix is an orthogonal matrix,
    and therefore its inverse is its transpose.
    Thus, using the notation used in (\ref{eq:permutation}), we also know that the
    characteristic polynomial of $\mathbf{A}$ is the same of $\mathbf{P} \mathbf{A} \mathbf{P}^{\mathrm{T}}$.
    This implies that the index $CI$ remains unchanged.
    \item[A3] The proof relies on a theorem from linear algebra by \citet{Kingman1961} stating that if the elements $a_{ij}(b)$ of a matrix $\mathbf{A}$
    are logconvex functions of $b$, then the maximum eigenvalue
    $\lambda_{\max}(\mathbf{A})$ is a logconvex (and hence convex) function of $b$.
    Given a pairwise comparison matrix $\mathbf{A}=(a_{ij}) \in \mathcal{A}$,
    let us apply the preference intensifying function $f(a_{ij})=
    a_{ij}^b$. The entries of the obtained matrix
    $\mathbf{A}(b) = (a_{ij}(b))_{n \times n} = (a_{ij}^{b})_{n \times n}$
     are logconvex functions of parameter $b$ (see \citealt{BozokiFulop2010}). From the
    theorem, $\lambda_{\max}(\mathbf{A}(b))$ is a convex function
    of $b$. Moreover, $\lambda_{\max}(\mathbf{A}(b))$ reaches its minimum value $n$ for $b=0$,
    since $\mathbf{A}(0)$ is the consistent matrix with all
    entries equal to one (indifference matrix). It follows that $\lambda_{\max}(\mathbf{A}(b))$
    is a non-decreasing function for $b \geq 0$ and hence, in
    particular, for $b \geq 1$. The same applies to $CI$, being an
increasing affine transform of $\lambda_{\max}(\mathbf{A}(b))$.
Therefore, $CI$ satisfies (\ref{A3}).
    \item[A4] 
It was proved by \citet{AupetitGenest1993}, that
$\lambda_{\max}$ must be either increasing, decreasing or U-shaped
as a function of a single upper triangular entry, say $a_{pq}$, of a
positive reciprocal matrix. Let \textbf{A} be a consistent matrix,
$\mathbf{A} \in \mathcal{A}^{*}$ and let $a_{pq} \neq 1$ be an
arbitrarily fixed entry of \textbf{A}. Then it is
$\lambda_{\max}(\textbf{A}) = n$. Using the notation of section
\ref{sec:A4}, let $\textbf{A}_{pq}(\delta)$ be the inconsistent
matrix obtained from \textbf{A} by replacing $a_{pq}$ with with a
different value $a_{pq}^{\delta} >0$, being $a_{qp}$ consequently
replaced with $1/a_{pq}^{\delta}$. Then $\lambda_{\max}$ becomes
larger than $n$, $\lambda_{\max}(\textbf{A}_{pq}(\delta)) \geq n$.
That is, $\lambda_{\max}(\textbf{A}_{pq}(\delta))$ reaches its
minimum value $n$ when $a_{pq}^{\delta} = a_{pq}$. It follows
that
$\lambda_{\max}(\textbf{A}_{pq}(\delta))$ is U-shaped as a function of
$a_{pq}^{\delta}$, reaching its minimum value for
$a_{pq}^{\delta}=a_{pq}$. The same applies to $CI$, being an
increasing affine transform of $\lambda_{\max}$, which proves that
A4 is satisfied.
\item[A5] Continuity holds due to the continuous dependence of the zeroes of a polynomial on its coefficients 
\end{description}
\end{proof}

\begin{proof}[Proof of Proposition \ref{prop:CI^{*}}]
We shall proceed by proving that it satisfies each axiom:

\begin{description}
    \item[A1] This axiom is satisfied by $CI^{*}$ as is can be proved that the non-negative quantity
    \[
    \frac{a_{ik}}{a_{ij}a_{jk}} + \frac{a_{ij}a_{jk}}{a_{ik}} - 2
    \]
    reaches its minimum, which is zero, if and only if $a_{ik}=a_{ij}a_{jk}$. The proof is based on the study of the function $G(x)=x+\frac{1}{x}-2$ where $x = \frac{a_{ik}}{a_{ij}a_{jk}}$. Considering $x>0$, then $G(x)$ reaches its minimum for $x=1$. This implies $a_{ik}=a_{ij}a_{jk}$.

    \item[A2]   Observe that: (i) expression (\ref{svil_det_3X3}) considers all the triplets $\{x_{i},x_{j},x_{k} \}$ only once; (ii) given a triplet $\{x_{i},x_{j},x_{k} \}$, the terms in (\ref{svil_det_3X3}) corresponding to these indices are independent from the order of the indices, that is
\begin{equation}
\label{eq:perm}
\frac{a_{ik}}{a_{ij}a_{jk}} + \frac{a_{ij}a_{jk}}{a_{ik}} = \frac{a_{\pi(i)\pi(k)}}{a_{\pi(i)\pi(j)}a_{\pi(j)\pi(k)}} + \frac{a_{\pi(i)\pi(j)}a_{\pi(j)\pi(k)}}{a_{\pi(i)\pi(k)}}
\end{equation}
holds for any permutation map $\pi$.

\item[A3]
It is possible to prove that each term of the sum
(\ref{svil_det_3X3}) corresponding to an inconsistent triplet increases if we apply $f(a_{ij})=a^{b}_{ij},~b>1$.
\noindent In fact, applying $f$ one obtains:
\begin{equation}
\label{termine somma}
\frac{a_{ik}^b}{a_{ij}^b a_{jk}^b }+\frac{a_{ij}^b a_{jk}^b
}{a_{ik}^b }-2 = \left( \frac{a_{ik}}{a_{ij} a_{jk} } \right)^b +
\left( \frac{a_{ij} a_{jk} }{a_{ik} } \right)^b -2.
\end{equation}

\noindent It is then necessary to prove that (\ref{termine somma}) is increasing with respect to $b>1$. Let us recall that $G(x)= x + \frac{1}{x}  -2$ and $x = \frac{a_{ik}}{a_{ij}a_{jk}}$. Let
\begin{equation}
g(x,b)= x^b + \frac{1}{x^b} -2 .
\end{equation}

\noindent We then need to prove that $ g(x,b) $ is increasing with respect to $b>1$, that is

\begin{equation}
b_2 > b_1 \Rightarrow g(x,b_2) \geq g(x,b_1)
\end{equation}

\noindent for all $x>0$ . It is

\begin{equation}
\label{derivata}
\frac{\partial g(x,b)}{\partial b} = \frac{\partial x^b}{\partial b}
+ \frac{\partial x^{-b}}{\partial b} = x^{b} \ln x -x^{-b} \ln x = (
\ln x ) \left(  x^{b}  - \frac{1}{x^b} \right) .
\end{equation}

\noindent To study the sign of $ \frac{\partial
g(x,b)}{\partial b} $, one can more simply analyze $ h(z)= z- \frac{1}{z}
$ in $]0,\infty[$, where $ z=x^{b}$

\begin{equation}
h(z)= z- \frac{1}{z} =  \frac{z^2 - 1}{z} .
\end{equation}

\noindent Then $0<z<1 \Rightarrow h(z)<0$, and $z>1 \Rightarrow
h(z)>0$. Consequently, for $b>1$,

\begin{eqnarray*}
0<x<1  & \Rightarrow & \ln x <0 \\
       & \Rightarrow &  0< z=x^b <1  \Rightarrow h(z)=\left(
       x^{b}  - \frac{1}{x^b} \right) < 0 \\
       & \Rightarrow & \frac{\partial g(x,b)}{\partial b} = ( \ln x ) \Big(  x^{b}  - \frac{1}{x^b}
       \Big)>0
\end{eqnarray*}

\begin{eqnarray*}
x>1  & \Rightarrow & \ln x >0 \\
       & \Rightarrow &  z=x^b >1  \Rightarrow h(z)=\left(
       x^{b}  - \frac{1}{x^b} \right) > 0 \\
       & \Rightarrow & \frac{\partial g(x,b)}{\partial b} = ( \ln x ) \left(  x^{b}  - \frac{1}{x^b}
       \right)>0 .
\end{eqnarray*}

\noindent Finally, if $x=1$, then

\begin{equation}
\frac{\partial g(1,b)}{\partial b} = 0 .
\end{equation}
\noindent Therefore, $g(x,b)$ is increasing respect to $b>0$ for
$x>0$ $x \neq 1$. For $x=1$, representing consistent triples, we get
$g(1,b)= 0 \;\; \forall b$.

\noindent To conclude, each term of the sum
(\ref{svil_det_3X3}) corresponding to an inconsistent triplet
increases by applying function $f$ with $b>1$, hence

\begin{eqnarray*}
CI^*(\mathbf{A})=0  & \Rightarrow &   CI^{*}(\mathbf{A}(b))=CI^{*}(\mathbf{A})=0 \\
CI^*(\mathbf{A})>0  & \Rightarrow &   CI^{*}(\mathbf{A}(b))>
CI^{*}(\mathbf{A}) .
\end{eqnarray*}

\item[A4] Without loss of generality, let us fix the entry $a_{pq}$ of $\mathbf{A}$ with $p<q$, and replace $a_{pq}$ with $a_{pq}^{\delta}$ and $a_{qp}$ with $a_{qp}^{b}$. Considering index (\ref{svil_det_3X3}) for the obtained matrix $\mathbf{A}_{pq}(\delta)$, some terms of the sum contain $a_{pq}^{\delta}$, while others do not. Those not containing $a_{pq}^{\delta}$ remain unchanged.
Let us first consider the following terms of the sum:

\begin{equation}
\label{eq:sum_delta} \frac{a_{pq}^{\delta}}{a_{pj}a_{jq}} +
\frac{a_{pj}a_{jq}}{a_{pq}^{\delta}} - 2 .
\end{equation}

Since $\mathbf{A}$ is consistent, we can set $x=a_{ik}=a_{ij}a_{jk}$ and rewrite (\ref{eq:sum_delta}) as a function of $x$ and $\delta$:
\[
H(x,\delta) = \frac{x^{\delta}}{x} + \frac{x}{x^{\delta}} - 2 =
x^{\delta-1} + x^{1-\delta} - 2 .
\]
It is:
\begin{equation*}
 \frac{\partial H(x,\delta)}{\partial \delta} =x^{\delta - 1} \ln{x}  -x^{1-\delta} \ln{x} = \underbrace{ \left. x^{1- \delta} \right. }_{>0} \underbrace{(x^{2\delta-2}-1)}_{\textrm{I}}
 \underbrace{\ln{x}}_{\textrm{II}} .
\end{equation*}
Thus, as pairwise comparison matrices are positive, i.e. $x>0$, one obtains the following four cases
\[
\left.
\begin{array}{r}
\delta >1  \\
x>1
\end{array}
\right\}
~\Rightarrow~
\left\{
\begin{array}{l}
\textrm{I} > 0  \\
\textrm{II} >0
\end{array}
\right.
~\Rightarrow ~
\frac{\partial H}{\partial \delta} > 0
\] \[
\left.
\begin{array}{r}
\delta >1  \\
0<x<1
\end{array}
\right\}
~\Rightarrow~
\left\{
\begin{array}{l}
\textrm{I} < 0  \\
\textrm{II} < 0
\end{array}
\right.
~\Rightarrow ~
\frac{\partial H}{\partial \delta} > 0
\] \[
\left.
\begin{array}{r}
\delta <1  \\
x>1
\end{array}
\right\}
~\Rightarrow~
\left\{
\begin{array}{l}
\textrm{I} < 0  \\
\textrm{II} > 0
\end{array}
\right.
~\Rightarrow ~
\frac{\partial H}{\partial \delta} < 0
\] \[
\left.
\begin{array}{r}
\delta <1  \\
0<x<1
\end{array}
\right\}
~\Rightarrow~
\left\{
\begin{array}{l}
\textrm{I} > 0  \\
\textrm{II} < 0
\end{array}
\right.
~\Rightarrow ~
\frac{\partial H}{\partial \delta} < 0 \, .
\]
Hence, every term of the sum (\ref{svil_det_3X3}) with the form
(\ref{eq:sum_delta}) is an increasing function of $\delta$ for
$\delta>1$ and decreasing function for $\delta<1$, thus satisfying
(\ref{monotonicity}). The other type of terms in the sum
(\ref{svil_det_3X3}) containing $a_{pq}^{\delta}$ are those in the
form
\begin{equation}
\label{eq:sum_delta2}
\frac{a_{pk}}{a_{pq}^{\delta}a_{qk}} + \frac{a_{pq}^{\delta}a_{qk}}{a_{pk}} - 2 \, .
\end{equation}
By means of a reasoning similar to the previous one, it can be proved that also (\ref{eq:sum_delta2}) is an increasing function of $\delta$ for $\delta>1$ and decreasing function for $\delta<1$. To summarize, A4 is satisfied by index $CI^{*}$.
\item[A5] For positive matrices, function $CI^{*}$ is continuous, as it is a sum of continuous functions.
\end{description}
\end{proof}

\begin{proof}[Proof of Proposition \ref{prop:GCI}] We shall prove each axiom separately.
\begin{description}
\item[A1] It is satisfied, since it is $a_{ij}=w_{i}/w_{j} \; \forall i, j$ if and only if $\mathbf{A} \in \mathcal{A}^{*} $, with $GCI(\mathbf{A})=\nu=0$.
\item[A2] The second axiom is verifiable by rewriting the index $GCI$ as
\[
GCI(\mathbf{A})=\frac{1}{(n-1)(n-2)}\sum_{i=1}^{n}\sum_{j=1}^{n} \,
\ln^2 \left( {a_{ij}\frac{w_j}{w_i}} \right) .
\]
\item[A3]
We consider $\mathbf{A}(b)=(a^{b}_{ij})_{n \times n}$ and $w_i (\mathbf{A}(b))$ be the $i$-th weight obtained from $\mathbf{A}(b)$ by means of (\ref{eq:mediageometrica}). Then, we obtain
\[
w_i (\mathbf{A}(b))=\Bigg( \prod_{i=1}^{n}a_{ij}^b \Bigg)
^\frac{1}{n} = \Bigg( \prod_{i=1}^{n}a_{ij} \Bigg) ^\frac{b}{n} =
(w_i (\mathbf{A}))^b ,
\]
\noindent so that, considering the terms of the sum (\ref{eq:GCI})
for the matrix $\mathbf{A}(b)$, we can derive the following
\[
 \ln^2 \left( {a_{ij}^{b}\frac{w_j(\mathbf{A}(b))}{w_i(\mathbf{A}(b))}} \right) =
  \ln^2\left( a_{ij} \frac{w_j(\mathbf{A})}{w_i(\mathbf{A})} \right)^b  = \\
  b^2 \ln^2 \left( {a_{ij}\frac{w_j}{w_i}} \right).
 \]
Hence,
\[
GCI(\mathbf{A}(b)) = b^2 GCI(\mathbf{A})
\]
\noindent and consequently A3 is satisfied, since
\[
GCI(\mathbf{A}(b)) \geq GCI(\mathbf{A}) ~~\forall b>1 .
\]

\item[A4]
As proved by \citet{BrunelliCritchFedrizzi2011}, the index $GCI$ is proportional to the following quantity
\begin{equation}
\label{eq:GCI_prop}
\sum_{i=1}^{n-2} \sum_{j=i+1}^{n-1} \sum_{k=j+1}^{n} \left(\log_{9}a_{ik}a_{kj}a_{ji}\right)^2.
\end{equation}
Given a consistent matrix $\mathbf{A} \in \mathcal{A}^{*}$, let us
fix the entry $a_{pq}~(\neq 1)$ of $\mathbf{A}$ with $p<q$. We
replace $a_{pq}$ with $a_{pq}^{\delta}$ and its reciprocal $a_{qp}$
with $a_{qp}^{\delta}$ obtaining the matrix
$\mathbf{A}_{pq}(\delta)$. Quantity (\ref{eq:GCI_prop}) is clearly
null for $\mathbf{A}$, while for $\mathbf{A}_{pq}(\delta)$ only the
terms not containing $a_{pq}^{\delta}$ or $a_{qp}^{\delta}$ are
null. Hence, let us consider the terms of (\ref{eq:GCI_prop}) for
$\mathbf{A}_{pq}(\delta)$ containing $a_{pq}^{\delta}$ (results can
be automatically extended to the terms containing
$a_{qp}^{\delta}$):
\[
\left( \log_{9}a^{\delta}_{pq}a_{qj}a_{jp} \right)^{2} .
\]
Since $\mathbf{A}$ is consistent, it is $a_{pq}=a_{pj}a_{jq}$. Like
in previous proofs, let us denote this quantity by $x$. Then, thanks
to reciprocity, it is $a_{qj}a_{jp}=1/x$ and thus
\[
\left( \log_{9}a_{pq}^{\delta}a_{qj}a_{jp} \right)^{2} = \left(
\log_{9}a_{pq}^{\delta -1}a_{pq}a_{qj}a_{jp} \right)^{2} = \left(
\log_{9} x^{\delta -1} \right)^{2} .
\]
Then every term in (\ref{eq:GCI_prop}) containing $a_{pq}^{\delta}$ has the same expression, say $\beta(x,\delta)=\left( \log_{9}x^{\delta -1} \right)^{2}= \left( (\delta -1) \log_{9} x \right)^{2}$. Taking the derivative with respect to $\delta$, it is,
\[
\frac{\partial \beta (x,\delta)}{\partial \delta} = 2(\delta -1)\left( \log_{9}x \right)^{2}
\]
and then
\[
\begin{array}{cc}
\delta > 1 &\Rightarrow \frac{\partial \beta (x,\delta)}{\partial \delta} >0 \\
\delta < 1 &\Rightarrow \frac{\partial \beta (x,\delta)}{\partial \delta} <0
\end{array}
\]
and the same holds also for the sum in (\ref{eq:GCI_prop}), thus proving that $GCI$ satisfies A4.
\item[A5] For positive matrices, function $GCI$ is continuous, as it is a sum of continuous functions.
\end{description}
\end{proof}


\begin{proof}[Proof of Proposition \ref{prop:RE_inv}]
For any $b \neq 0$ it is
\begin{align*}
RE(\mathbf{A}(b))&=1-\dfrac{ \sum_{i=1}^{n} \sum_{j=1}^{n} \left( \frac{1}{n} \sum_{k=1}^{n} \log a_{ik}^{b}- \frac{1}{n} \sum_{k=1}^{n} \log a_{jk}^{b} \right)^{2}}{ \sum_{i=1}^{n} \sum_{j=1}^{n} (\log a_{ij}^{b})^{2}}\\
&=1-\dfrac{ \sum_{i=1}^{n} \sum_{j=1}^{n} b^{2} \left( \frac{1}{n} \sum_{k=1}^{n}  \log a_{ik}- \frac{1}{n} \sum_{k=1}^{n}  \log a_{jk} \right)^{2}}{ \sum_{i=1}^{n} \sum_{j=1}^{n} b^{2} (\log a_{ij})^{2}}\\
&=1-\dfrac{  \sum_{i=1}^{n} \sum_{j=1}^{n} \left( \frac{1}{n} \sum_{k=1}^{n} \log a_{ik}- \frac{1}{n} \sum_{k=1}^{n} \log a_{jk} \right)^{2}}{  \sum_{i=1}^{n} \sum_{j=1}^{n} (\log a_{ij})^{2}}\\
&=RE(\mathbf{A}).
\end{align*}
\end{proof}

\begin{proof}[Proof of Proposition \ref{prop:RE}]
In the original paper by \citet{Barzilai1998}, it was
remarked that A1 is satisfied by $RE$. The proof of the satisfaction
of A2 is also elementary and Proposition \ref{prop:RE_inv} proves
the satisfaction of A3. In the following we shall prove separately
that A4 and A5 do not hold.
\begin{description}
\item[A4]
Following \citet{Barzilai1998}, let us consider the
equivalent additive formulation of the pairwise comparison matrices
and rewrite $RE$ in the form

\begin{equation}
RE(\mathbf{A})=\frac{ \sum_{i=1}^{n} \sum_{j=1}^{n} e_{ij}} {
\sum_{i=1}^{n} \sum_{j=1}^{n} a_{ij}} ,
\end{equation}
\noindent where $e_{ij}  = a_{ij} - c_{ij} , \quad c_{ij}  = w_i -
w_j $ and $ w_i  = \frac{1}{n}\sum_{j=1}^{n} a_{ij} $ .
We start from a consistent (in the additive sense) matrix
$\mathbf{A}$ and replace the entry $a_{pq}$ with $a_{pq}' = a_{pq} +
x $, $x \neq 0$. Necessarily, $a_{qp}$ must be replaced by $a_{qp}'
= a_{qp}-x$ to preserve the additive reciprocity, i.e. the
antisymmetry of $\mathbf{A}$. Let
$\mathbf{A}'=\mathbf{A}'(x)=(a_{ij}')$ be the obtained inconsistent
matrix. In order to evaluate $RE(\mathbf{A}')$, let us calculate
\begin{align*}
w_i' & = w_i  \quad \quad \textrm{if} \quad i \neq p , i \neq q \\
w_p' & =  \frac{1}{n}\sum_{j=1}^{n} a_{pj}' = \frac{1}{n}\sum_{j=1}^{n} a_{pj} + \frac{x}{n} = w_p + \frac{x}{n} \\
w_q' & =  \frac{1}{n}\sum_{j=1}^{n} a_{qj}' = \frac{1}{n}\sum_{j=1}^{n} a_{qj} - \frac{x}{n} = w_q - \frac{x}{n} \\
c_{ij}' & = c_{ij}  \quad \quad \textrm{if} \quad i, j \neq p , q \\
c_{pj}' & = w_p' - w_j' =  w_p + \frac{x}{n} - w_j = c_{pj} +
\frac{x}{n} \quad \textrm{if} \quad j \neq p,  j \neq q \\
c_{pq}' & = w_p' - w_q' =  w_p + \frac{x}{n} - (w_q - \frac{x}{n}) =
w_p - w_q + \frac{2x}{n} = c_{pq} + \frac{2x}{n} \\
c_{ip}' & = w_i' - w_p' =  w_i - w_p - \frac{x}{n}  = c_{ip} -
\frac{x}{n} \quad \textrm{if} \quad i \neq p,  j \neq q \\
c_{qj}' & = w_q' - w_j' =  w_q - \frac{x}{n} - w_j   = c_{qj} -
\frac{x}{n} \quad \textrm{if} \quad j \neq p,  j \neq q \\
c_{qp}' & = w_q' - w_p' =  w_q - \frac{x}{n} - (w_p + \frac{x}{n}) =
w_q - w_p - \frac{2x}{n} = c_{qp} - \frac{2x}{n} \\
c_{iq}' & = w_i' - w_q' =  w_i - w_q + \frac{x}{n}  = c_{iq} +
\frac{x}{n} \quad \textrm{if} \quad i \neq p,  i \neq q .
\end{align*}
\noindent Therefore, the values $e_{ij}'$ are
\begin{align*}
e_{ij}' & = a_{ij}' - c_{ij}' =0  \quad \quad \textrm{if} \quad i, j \neq p , q \\
e_{pj}' & = a_{pj}' - c_{pj}' = a_{pj} -(c_{pj} + \frac{x}{n})=
e_{pj} - \frac{x}{n} = - \frac{x}{n} \quad \quad \textrm{if} \quad j
\neq p, q \\
e_{pq}' & = a_{pq}' - c_{pq}' = a_{pq} + x -(c_{pq} + \frac{2x}{n})=
e_{pq} + x - \frac{2x}{n} = x - \frac{2x}{n} = \frac{n-2}{n}x \\
e_{qj}' & = a_{qj}' - c_{qj}' = a_{qj} -(c_{qj} - \frac{x}{n})=
e_{qj} + \frac{x}{n} = \frac{x}{n} \quad \quad \textrm{if} \quad j
\neq p, q \\
e_{qp}' & = a_{qp}' - c_{qp}' = a_{qp} - x -(c_{qp} - \frac{2x}{n})=
e_{qp} - x + \frac{2x}{n} = - \frac{n-2}{n}x \\
e_{ip}' & = a_{ip}' - c_{ip}' = a_{ip} -(c_{ip} - \frac{x}{n})=
e_{ip} + \frac{x}{n} = \frac{x}{n} \quad \quad \textrm{if} \quad i
\neq p, q \\
e_{iq}' & = a_{iq}' - c_{iq}' = a_{iq} -(c_{iq} + \frac{x}{n})=
e_{iq} - \frac{x}{n} = - \frac{x}{n} \quad \quad \textrm{if} \quad i
\neq p, q ,
\end{align*}
\noindent and, clearly,
\[
e_{ii}' =0 \qquad i=1,...,n .
\]
\noindent The relative error $RE(\mathbf{A}')$ is then
\begin{align*}
RE(\mathbf{A}') & =\frac{ \sum_{i=1}^{n} \sum_{j=1}^{n} (e_{ij}')^2}
{ \sum_{i=1}^{n} \sum_{j=1}^{n} (a_{ij}')^2} \\
& = \frac{ \sum_{j \neq q, p} (e_{pj}')^2 + (e_{pq}')^2 + \sum_{j
\neq p, q} (e_{qj}')^2 + (e_{qp}')^2 + \sum_{i \neq p, q}
(e_{ip}')^2 + \sum_{i \neq p, q} (e_{iq}')^2}
 {\sum_{ij} (a_{ij}')^2} \\
& = \frac{ \sum_{j \neq p, q} (- \frac{x}{n})^2 + (\frac{n-2}{n}x)^2
+ \sum_{j \neq p, q} (\frac{x}{n})^2 + (- \frac{n-2}{n}x)^2 +
\sum_{i \neq p, q} (\frac{x}{n})^2 + \sum_{i \neq p, q} (-
\frac{x}{n})^2}
 {\sum_{ij} (a_{ij})^2 - (a_{pq})^2 - (a_{qp})^2 + (a_{pq} + x)^2 + (a_{qp} -
 x)^2} \\
  & = \frac{ 4(n-2)(\frac{x}{n})^2 + 2(\frac{n-2}{n}x)^2 }
 {\sum_{ij} (a_{ij})^2 + 2a_{pq}x + x^2 - 2a_{qp}x + x^2} \\
 & = \frac{2(n^2-2n)}{n^2} \frac{ x^2 }  {2x^2 +  4a_{pq}x + \sum_{ij}(a_{ij})^2 } \\
 & = H_n \frac{ x^2 }  {x^2 +  2 \alpha x + K } ,
\end{align*}
\noindent where
\[
H_n  =  \frac{(n^2-2n)}{n^2} = 1 - \frac{2}{n} \, ; \qquad K  =
\frac{1}{2} \sum_{ij}(a_{ij})^2 \, ; \qquad \alpha  = a_{pq}
\]
\noindent By disregarding the positive constant $H_n$ (for $n \geq
3$) and taking the derivative of $RE(\mathbf{A}')$ with respect to
$x$ one obtains
\begin{align*}
\frac{\partial RE(\mathbf{A}')}{\partial x} & = \frac{2x(x^2 + 2
\alpha x + K) - x^2(2x + 2 \alpha)}{(x^2 +  2 \alpha x + K)^2} \\
& = 2 \frac{ (2 \alpha x + (K -  \alpha))x}{(x^2 +  2
\alpha x + K)^2} . \\
\end{align*}
\noindent Assuming, without loss of generality, $\alpha < 0$, it
follows $\frac{\alpha - K}{2   \alpha} >0 $. Then, $RE(\mathbf{A}')$
 decreases in $(-\infty , 0) $ , increases in $(0 , \frac{\alpha - K}{2   \alpha})
 $ and decreases in $(\frac{\alpha - K}{2   \alpha} , + \infty). $
 This means that if $a_{pq}$ is increased by $x$, the inconsistency
 index $RE(\mathbf{A}')$ does not monotonically increases in $(0 , + \infty) $, so that
 A4 is not satisfied.

\item[A5] From Proposition \ref{prop:RE_inv}, it follows that
$RE(\mathbf{A})=RE(\mathbf{A}(b))~\forall b \neq 0$. Then, for every
$\mathbf{A} \in \mathcal{A}$ it holds
\begin{equation}
\label{lim1}
\lim_{b\rightarrow 0} RE(\mathbf{A}(b))=RE(\mathbf{A}) \, .
\end{equation}
\noindent Conversely, it is
\begin{equation}
\label{lim2}
RE(\lim_{b\rightarrow 0} \mathbf{A}(b)) = RE((1)_{n \times n})=0 \,
,
\end{equation}
\noindent and therefore
\begin{equation}
\lim_{b\rightarrow 0} RE(\mathbf{A}(b)) \neq RE(\lim_{b\rightarrow
0} \mathbf{A}(b))
\end{equation}
\noindent for every inconsistent matrix $\mathbf{A}$. Then,
$RE(\mathbf{A})$ is not continuous in $(1)_{n \times n}$.
\noindent Note that, since (\ref{lim1}) and (\ref{lim2}) hold for
every matrix $\mathbf{A} \in \mathcal{A}$, then, in every
neighborhood of the consistent matrix $(1)_{n \times n}$ there exist
matrices $\mathbf{A}$ with any possible value of $RE(\mathbf{A})$,
even the highest one $RE(\mathbf{A})=1$.
\end{description}

\end{proof}

\begin{proof}[Proof of Proposition \ref{prop:HCI}] We shall prove each axiom separately.
\begin{description}
\item[A1] This comes along with the result proved by \citet{SteinMizzi2007}
that $HCI(\mathbf{A})=0$ if and only if $\mathbf{A}$ is
consistent.
\item[A2] A permutation of the alternatives corresponds to a row-column permutation on the pairwise comparison matrix. When
columns of $\mathbf{A}$ are swapped this induces a permutation of
the indices of $s_{1},\ldots,s_{n}$. When rows are swapped, this
produces a change in the order of the arguments of sums
$\sum_{i=1}^{n} a_{ij}~\forall j$. As both the harmonic mean
(\ref{eq:mediaarmonica}) and the sum are commutative functions, A2 is satisfied.
\item[A3] Let us
prove it for a subset of inconsistent pairwise comparison matrices.
We consider all the inconsistent pairwise comparison matrices which
have one column, say $j^{*}$, with all entries smaller than one
except for the diagonal element $a_{j^{*}j^{*}}$, i.e. $\exists
j^{*}, a_{ij^{*}} < 1 \forall i \neq j^{*}$. In this case, if we
apply transformation $f(a_{ij})=a_{ij}^{b}$ and let $b$ increase to
$+\infty$, we obtain that all the values $s_{j}$ tend to infinite
except $s_{j^{*}}$ which tends to 1. Consequently, all the terms of
the sum $\sum_{j=1}^{n}\frac{1}{s_{j}}$ converge to 0, except
$1/s_{j^{*}}$ which, instead, converges to 1, showing that
$HM(\mathbf{A})$ will tend to $n$ and $HCI(\mathbf{A})$ to 0. As the
initial set of matrices was inconsistent and therefore their value
of $HCI$ must have been positive, A3 is not satisfied.

\item[A4] Consider that a consistent pairwise comparison matrix of order $n$ can be written as
\begin{equation}
\mathbf{A}=
\begin{pmatrix}
1                                      & a_{12}                                & a_{12}a_{23}     & \cdots  & a_{12} \cdots a_{n-1 \, n}  \\
\frac{1}{a_{12}}                       &    1                                  &     a_{23}       & \cdots  & a_{23} \cdots a_{n-1 \, n}  \\
\cdots                                 & \cdots                                & \cdots                             & \cdots  & \cdots    \\
\frac{1}{a_{12} \cdots a_{n-2 \, n-1}} & \frac{1}{a_{23} \cdots a_{n-2 \, n-1}}& \frac{1}{a_{34} \cdots a_{n-2 \, n-1}}                                  & \cdots  & a_{n-1 \, n} \\
\frac{1}{a_{12} \cdots a_{n-1 \, n}}   & \frac{1}{a_{23} \cdots a_{n-1 \, n}}  & \frac{1}{a_{34} \cdots a_{n-1 \, n}}& \cdots  & 1
\end{pmatrix} \in \mathcal{A}^{*}
\end{equation}
\noindent Let us apply the exponential function to the element $a_{12}$ ($\neq 1$)
and its reciprocal:
\begin{equation}
\mathbf{A}_{12}(\delta)=
\begin{pmatrix}
1                                      & (a_{12})^{\delta}                                & a_{12}a_{23}     & \cdots  & a_{12} \cdots a_{n-1 \, n}  \\
\left(\frac{1}{a_{12}}\right)^{\delta}                       &    1                                  &     a_{23}       & \cdots  & a_{23} \cdots a_{n-1 \, n}  \\
\cdots                                 & \cdots                                & \cdots                             & \cdots  & \cdots    \\
\frac{1}{a_{12} \cdots a_{n-2 \, n-1}} & \frac{1}{a_{23} \cdots a_{n-2 \, n-1}}& \frac{1}{a_{34} \cdots a_{n-2 \, n-1}}                                  & \cdots  & a_{n-1 \, n} \\
\frac{1}{a_{12} \cdots a_{n-1 \, n}}   & \frac{1}{a_{23} \cdots a_{n-1 \, n}}  & \frac{1}{a_{34} \cdots a_{n-1 \, n}}& \cdots  & 1
\end{pmatrix}
\end{equation}
\noindent In fact, as index $HCI$ satisfies A2, by choosing
$a_{12}$ we do not lose generality. To prove that $HCI$ satisfies
A4 we shall show that
\begin{eqnarray}
  \frac{\partial HCI((\mathbf{A}_{12}(\delta))}{\partial \delta} &<& 0 \quad \mathrm{for} \quad
\delta<1 \label{kminore1} \\
  \frac{\partial HCI((\mathbf{A}_{12}(\delta))}{\partial \delta} &>& 0 \quad \mathrm{for} \quad
\delta>1 \label{kge1}
\end{eqnarray}
\noindent It is sufficient to prove it for
$HM(\mathbf{A}_{12}(\delta))$, since $HCI(\mathbf{A}_{12}(\delta))$ is just
one of its monotone increasing affine transforms.
For notational convenience, let us define $b_{pq}=\prod_{i=p}^{q-1}a_{i \, i+1} $, e.g.
$b_{14}=a_{12}a_{23}a_{34}$. Thus, we can rewrite
$HM(\mathbf{A}_{12}(\delta))$ in the following way, to keep explicit
$\left( a_{12} \right)^{\delta}$ and $\left( 1 / a_{12}
\right)^{\delta}$
\begin{equation}
\label{eq:HMdelta}
HM(\mathbf{A}_{12}(\delta))=\frac{n}
{\underbrace{\frac{1}{1+\left(\frac{1}{a_{12}}\right)^{\delta}+\sum_{j=3}^{n} \frac{1}{b_{1j}}}}_{\frac{1}{s_{1}}}
+
\underbrace{\frac{1}{(a_{12})^{\delta} + 1 + \sum_{j=3}^{n} \frac{1}{b_{2j}} }}_{\frac{1}{s_{2}}}
+
\sum_{i=3}^{n} \frac{1}{s_{i}}}.
\end{equation}
Considering that only the first two arguments of the sum at the
denominator of (\ref{eq:HMdelta}) contain $\delta$, its derivative
w.r.t. $\delta$ is
\begin{equation}
\frac{\partial HM(\mathbf{A}_{12}(\delta))}{\partial \delta} = - \frac{n \left[
\frac{a_{12}^{-\delta}\ln
( a_{12})}{\left(1+\left(\frac{1}{a_{12}}\right)^\delta + \sum_{j=3}^{n} \frac{1}{b_{1j}} \right)^2}
- \frac{a_{12}^{\delta}\ln (a_{12})}{\left( (a_{12})^\delta + 1 +
\sum_{j=3}^{n} \frac{1}{b_{2j}} \right)^2}\right]
}{ \left( \cdot \right)^2} \label{eq2}
\end{equation}
We first consider the case $\delta>1$ and $a_{12}>1$, which implies
$\ln(a_{12})>0$. Condition (\ref{kge1}) is satisfied if in
(\ref{eq2}) the quantity between square brackets is negative, i.e.
\begin{equation}
\frac{a_{12}^{-\delta}}{\left(1+\left(\frac{1}{a_{12}}\right)^\delta + \sum_{j=3}^{n} \frac{1}{b_{1j}} \right)^2}
<
 \frac{a_{12}^{\delta}}{\left( (a_{12})^{\delta} + 1 + \sum_{j=3}^{n} \frac{1}{b_{2j}} \right)^2} .%
\label{disug1}
\end{equation}
\noindent With some computation one obtains
\[
\frac{a_{12}^{-\delta}}{\left(\frac{ b_{1n} + a_{12}^{1-\delta} b_{2n} + \sum_{j=3}^{n-1}b_{jn} + 1 }{b_{1n}} \right)^2}
<
\frac{\frac{a_{12}^{\delta}}{a_{12}^{2}}}{\left( \frac{ a_{12}^{\delta}b_{2n} + \sum_{j=2}^{n-1}b_{jn} +1  }{b_{1n}}
\right)^2}
\]
\[
\frac{1}{\left(b_{1n} + a_{12}^{1-\delta} b_{2n} + \sum_{j=3}^{n-1}b_{jn} + 1  \right)^2}
<
\frac{a_{12}^{2\delta-2}}{\left(  a_{12}^{\delta}b_{2n} +
\sum_{j=2}^{n-1}b_{jn} +1 \right)^2}
\]
\[
b_{1n} + a_{12}^{1-\delta} b_{2n} + \sum_{j=3}^{n-1}b_{jn} + 1 %
>
\frac{  a_{12}^{\delta}b_{2n} + \sum_{j=2}^{n-1}b_{jn} +1
}{a_{12}^{\delta-1}} .
\]
By multiplying both sides times $a_{12}^{\delta-1}$ and given that
$b_{1n}=a_{12}b_{2n}$, it is
\[
b_{2n} + a_{12}^{\delta -1} \sum_{j=3}^{n-1} b_{jn} + a_{12}^{\delta-1}
>
\sum_{j=2}^{n-1} b_{jn} +1 .
\]
Given that $\left( \sum_{j=2}^{n-1} b_{jn} \right) - b_{2n} = \sum_{j=3}^{n-1} b_{jn}$, then
\[
a_{12}^{\delta-1} \left( \sum_{j=3}^{n-1} b_{jn} + 1 \right) > \left( \sum_{j=3}^{n-1} b_{jn} + 1 \right) \Rightarrow
a_{12}^{\delta-1} > 1.
\]
\noindent Thus, if  $a_{12}>1$ and $\delta>1$, one obtains that
$\frac{\partial HM((\mathbf{A}_{12}(\delta))}{\partial \delta} > 0 $
and $\frac{\partial HCI((\mathbf{A}_{12}(\delta))}{\partial \delta}
> 0 $.
\noindent If, instead, $0<a_{12}<1$, one knows that
$\ln(a_{12})<0$, and all inequalities, starting from
(\ref{disug1}), are inverted. Consequently, $ a_{12}^{\delta-1} <
1 $, which is true for $\delta>1$, and so, also in this case,
$\frac{\partial HCI((\mathbf{A}_{12}(\delta))}{\partial \delta} > 0$.
For $\delta<1$, the opposite happens: condition (\ref{kminore1}) is
satisfied and $\frac{\partial
HCI((\mathbf{A}_{12}(\delta))}{\partial \delta} < 0 $. This proves
that $HCI$ satisfies (\ref{monotonicity}).
\item[A5] Since pairwise comparison matrices are positive matrices,
it follows that $s_j>0 \; \forall j$. Therefore, $HCI(\mathbf{A})$
is a continuous function.
\end{description}
\end{proof}

\begin{proof}[Proof of Proposition \ref{prop:GW}]
As in the other proofs, we shall prove the axioms in sequence.
\begin{description}
\item[A1 A2]
The proof that $GW$ satisfies axioms A1 and A2 is straightforward.
\item[A3]
Analogously to Proposition \ref{prop:HCI}, we shall prove that $GW$
does not satisfy A3 for a whole set of pairwise comparison matrices,
instead of providing a single counterexample. Consider the set of
inconsistent pairwise comparison matrices $\mathbf{A}$ such that all
the greatest elements of each column lie on the same row, i.e.
$\exists i^{*}, a_{i^{*}j}>a_{ij} \; \forall i\neq i^{*}, \; \forall
j$. It is $GW(\mathbf{A})>0$ for all these matrices, since they are
inconsistent. Without loss of generality we assume $i^{*}=1$; then,
applying $f$ to $\mathbf{A}$ obtaining
$f(\mathbf{A})=\mathbf{A}(b)=(a_{ij}^{b})_{ n \times n}$, when the
exponent $b$ tends to $+\infty$, ($b \rightarrow +\infty$), it is
\[
\overline{\mathbf{A}(b)} \rightarrow \begin{pmatrix}
1 & 1 & \ldots & 1 \\
0 & 0 & \ldots & 0 \\
\ldots & \ldots & \ldots & \ldots \\
0 & 0 & \ldots & 0 \\
\end{pmatrix}~~~~
\bar{\mathbf{g}} \rightarrow \begin{pmatrix}
1   \\
0   \\
\vdots  \\
0  \\
\end{pmatrix};
\]
where $\bar{\mathbf{g}}$ is the normalized weight vector obtained
from $\overline{\mathbf{A}(b)}$ by means of the geometric mean
method. Hence index $GW$ tends to $0$, ($GW(\mathbf{A}(b))
\rightarrow 0$) and violates (\ref{A3}).
\item[A5] Index $GW$ is continuous, as it is a sum of continuous functions.
\end{description}
\end{proof}

\begin{thebibliography}{999}

\footnotesize


\bibitem[{Aczel(1966)}]{Aczel1966}
         Aczel J (1966).
         \emph{Lectures on Functional Equations and their Applications}.
         Academic Press: New York.

\bibitem[{Aguar\`on and Moreno-Jim\`enez(2003)}]{AguaronMoreno2003}
         Aguar\`on, J and Moreno-Jim\`enez, J~M (2003).
         The geometric consistency index: Approximated threshold.
         \emph{European Journal of Operational Research}
         \textbf{147}(1): 137--145.

\bibitem[{Aupetit and Genest(1993)}]{AupetitGenest1993}
         Aupetit B and Genest C (1993).
         On some useful properties of the Perron eigenvalue of a positive reciprocal matrix in the context of the analytic hierarchy process.
         \emph{European Journal of Operational Research}
         \textbf{70}(2): 263--268

\bibitem[{Barzilai(1998)}]{Barzilai1998}
                 Barzilai J (1998).
                 Consistency measures for pairwise comparison matrices.
                 \emph{Journal of Multi-Criteria Decision Analysis}
                 \textbf{7}(3): 123--132.


\bibitem[{Boz\'{o}ki S. and Rapcs\'{a}k(2008)}]{BozokiRapcsak2008}
                 Boz\'{o}ki S and Rapcs\'{a}k T (2008).
                 On Saaty's and Koczkodaj's inconsistencies of pairwise comparison matrices.
                 \emph{Journal of Global Optimization}
                 \textbf{42}(2): 157--175.

\bibitem[{Boz\'{o}ki et al.(2010)}]{BozokiFulop2010}
                 Boz\'{o}ki S, F{\"u}l{\"o}p J and R{\'o}nyai L (2010).
                 On optimal completion of incomplete pairwise comparison matrices.
                 \emph{Mathematical and Computer Modelling}
                 \textbf{52}(1--2): 318--333.

\bibitem[{Brunelli(2011)}]{Brunelli2011}
        Brunelli M (2011).
        A note on the article ``Inconsistency of pair-wise comparison matrix with fuzzy elements based on geometric mean'' [Fuzzy Sets and Systems 161 (2010) 1604--1613].
        \emph{Fuzzy Sets and Systems}
        \textbf{176}(1): 76--78.

\bibitem[Brunelli et al.(2013a)]{BrunelliCanalFedrizzi}
        Brunelli M, Canal L and Fedrizzi M (2013a).
        Inconsistency indices for pairwise comparison matrices: a numerical study.
        \emph{Annals of Operations Research}
        doi: 10.1007/s10479-013-1329-0.


\bibitem[Brunelli et al.(2013b)]{BrunelliCritchFedrizzi2011}
        Brunelli M, Critch A and Fedrizzi M (2013b).
        A note on the proportionality between some consistency indices in the AHP.
        \emph{Applied Mathematics and Computation}
        \textbf{219}(14): 7901--7906.

\bibitem[{Bryson(1995)}]{Bryson1995}
        Bryson N (1995).
        A goal programming method for generating priority vectors.
        \emph{Journal of the Operational Research Society}
        \textbf{46}(5): 641--648.

\bibitem[{Cavallo and D'Apuzzo(2009)}]{CavalloD'Apuzzo2009}
         Cavallo B and D'Apuzzo L (2009).
         A general unified framework for pairwise comparison matrices in multicriterial methods.
         \emph{International Journal of Intelligent Systems}
         \textbf{24}(4): 377--398.



\bibitem[{Choo and Wedley(2004)}]{ChooWedley2004}
         Choo E~U and Wedley W~C (2004).
         A common framework for deriving preference values from pairwise comparison matrices.
         \emph{Computers \& Operations Research}
         \textbf{31}(6): 893--908.


\bibitem[{Cook and Kress(1988)}]{CookKress1988}
        Cook W~D and Kress M (1988).
        Deriving weights from pairwise comparison ratio matrices: An axiomatic approach.
        \emph{European Journal of Operational Research}
        \textbf{37}(3): 355--362.


\bibitem[{Crawford and Williams(1985)}]{CrawfordWilliams1985}
         Crawford G and Williams C (1985).
         A note on the analysis of subjective judgement matrices.
         \emph{Journal of Mathematical Psychology}
         \textbf{29}(4): 25--40.

\bibitem[{Duszak and Koczkodaj(1994)}]{DuszakKoczkodaj1994}
         Duszak Z and Koczkodaj W~W (1994).
         Generalization of a new definition of consistency for pairwise comparisons.
         \emph{Information Processing Letters}
         \textbf{52}(5): 273--276.

\bibitem[{Fedrizzi and Brunelli(2009)}]{FedrizziBrunelli2009}
         Fedrizzi M and Brunelli M (2009).
         Fair consistency evaluation in reciprocal relations and group decision making.
         \emph{New Mathematics and Natural Computation}
         \textbf{5}(2): 407--420.




\bibitem[{Golden and Wang(1989)}]{GoldenWang1989}
                 Golden B~L and Wang Q (1989).
                 An alternate measure of consistency.
                 In: Golden B~L., Wasil E~A and Harker P~T (eds),
                 \emph{The Analytic Hierarchy Process, Applications and studies},
                 (pp. 68--81), Springer-Verlag: Berlin--Heidelberg.

\bibitem[{Harker(1987)}]{Harker1987}
                Harker P~T (1987).
                Incomplete pairwise comparisons in the analytic hierarchy process.
                \emph{Mathematical Modelling}
                \textbf{9}(11): 837--848.

\bibitem[{Harker and Vargas(1987)}]{HarkerVargas1987}
                Harker P~T, Vargas L~G (1987).
                The theory of ratio scale estimation: Saaty's Analytic Hierarchy Process.
                \emph{Management Science}
                \textbf{33}(11): 1383--1403.

\bibitem[{Herrera-Viedma et al.(2004)}]{Herrera-Viedma2004}
         Herrera-Viedma E, Herrera F, Chiclana F and Luque M (2004).
         Some issues on consistency of fuzzy preference relations.
         \emph{European Journal of Operational Research}
         \textbf{154}(1): 98--109.


\bibitem[{Horn and Johnson(1985)}]{HornJohnson1985}
         Horn R~A and Johnson C~R (1985).
         \emph{Matrix Analysis},
         Cambridge University Press: New York.

\bibitem[{Irwin(1958)}]{Irwin1958}
         Irwin F~W (1958).
         An analysis of the concepts of discrimination and preference.
         \emph{The American Journal of Psychology}
         \textbf{71}(1): 152--163.

\bibitem[{Ishizaka et al.(2011)}]{IshizakaEtAl2011}
         Ishizaka A (2011).
         Does AHP help us make a choice? An experimental evaluation.
         \emph{Journal of the Operational Research Society}
         \textbf{62}(10): 1801--1812.

\bibitem[{Ishizaka and Labib(2011)}]{IshizakaLabib2011}
         Ishizaka A and Labib A (2011).
         Review of the main developments in the analytic hierarchy process.
         \emph{Expert Systems with Applications}
         \textbf{38}(11): 14336--14345.

\bibitem[{Kemeny and Snell(1962)}]{KemenySnell1962}
         Kemeny, J~G and Snell J~L (1962).
         \emph{Mathematical Models in the Social Sciences}.
         Blaisdell: New York.

\bibitem[{Kingman(1961)}]{Kingman1961}
         Kingman, J~F~C (1961).
         A convexity property of positive matrices.
         \emph{The Quarterly Journal of Mathematics. Oxford. Second Series}
         \textbf{12}(1), 283--284.

\bibitem[{Koczkodaj(1993)}]{Koczkodaj1993}
         Koczkodaj W~W (1993).
         A new definition of consistency of pairwise comparisons.
         \emph{Mathematical and Computer Modelling}
         \textbf{18}(7): 79--84.

\bibitem[{Lamata and Pel\'{a}ez(2002)}]{LamataPelaez2002}
         Lamata M~T and Pel\'{a}ez J~I (2002).
         A method for improving the consistency of judgments.
         \emph{International Journal of Uncertainty, Fuzziness and Knowledge-Based Systems}
         \textbf{10}(6): 677--686.




\bibitem[{Osei-Bryson(2006)}]{OseiBryson2006}
        Osei-Bryson N (2006).
        An action learning approach for assessing the consistency of pairwise comparison data.
        \emph{European Journal of Operational Research}
        \textbf{174}(1): 234--244.

\bibitem[{Pel\'{a}ez and Lamata(2003)}]{PelaezLamata2003}
         Pel\'{a}ez J~I and Lamata M~T (2003).
         A new measure of inconsistency for positive reciprocal matrices.
         \emph{Computer and Mathematics with Applications}
         \textbf{46}(12): 1839--1845.

\bibitem[{Ram\'{i}k and Korviny(2010)}]{RamikKorviny2010}
         Ram\'{i}k J and Korviny P (2010).
         Inconsistency of pair-wise comparison matrix with fuzzy elements based on geometric mean.
         \emph{Fuzzy Sets and Systems}
         \textbf{161}(11): 1604--1613.


\bibitem[{Saaty(1977)}]{Saaty1977}
         Saaty T~L (1977).
         A scaling method for priorities in hierarchical structures.
         \emph{Journal of Mathematical Psychology}
         \textbf{15}(3): 234--281.


\bibitem[{Saaty(1993)}]{Saaty1993}
                 Saaty T~L (1993).
                 What is relative measurement? The ratio scale phantom.
                 \emph{Mathematical and Computer Modelling}
                 \textbf{17}(4--5): 1--12.


\bibitem[{Saaty(1994)}]{Saaty1994}
                 Saaty T~L (1994).
                 Highlights and critical points in the theory and application of the Analytic Hierarchy Process.
                 \emph{European Journal of Operational Research}
                 \textbf{74}(3): 426--447.

\bibitem[{Salo(1993)}]{Salo1993}
        Salo A~A (1993).
        Inconsistency analysis by approximately specified priorities.
        \emph{Mathematical and Computer Modelling}
        \textbf{17}(4--5): 123--133.

\bibitem[{Shiraishi et al.(1998)}]{ShiraishiEtAl1998}
                 Shiraishi S, Obata T and Daigo M (1998).
                 Properties of a positive reciprocal matrix and their application to AHP.
                 \emph{Journal of the Operations Research Society of Japan}
                 \textbf{41}(3): 404--414.

\bibitem[{Shiraishi et al.(1999)}]{ShiraishiEtAl1999}
                 Shiraishi S, Obata T, Daigo M and Nakajima N (1999).
                 Assessment for an incomplete matrix and improvement of the inconsistent comparison: computational experiments.
                 \emph{Proceedings of ISAHP 1999}, Kobe, Japan


\bibitem[{Stein and Mizzi(2007)}]{SteinMizzi2007}
                 Stein W~E and Mizzi P~J (2007).
                 The harmonic consistency index for the analytic hierarchy process.
                 \emph{European Journal of Operational Research}
                 \textbf{177}(1): 488--497.




\bibitem[{Xu and Cuiping(1999)}]{XuCuiping1999}
                 Xu Z and Cuiping W (1999).
                 A consistency improving method in the analytic hierarchy process.
                 \emph{European Journal of Operational Research}
                 \textbf{116}(2):  443--449.

\bibitem[{Xu and Xia(2013)}]{XuXia2013}
                 Xu Z and Xia M (1999).
                 Iterative algorithms for improving consistency of intuitionistic preference relations.
                 \emph{Journal of Operational Research Society}
                 doi: 10.1057/jors.2012.178.






\end{thebibliography}
\end{document}